\documentclass[12pt,reqno,twoside]{amsart}
\usepackage[margin=1in]{geometry}
\usepackage{amsmath, amssymb, amsthm}
\usepackage[dvipsnames,table,xcdraw]{xcolor}

\usepackage[final]{graphicx} 
\usepackage{subcaption} 
\usepackage{hyperref}
\hypersetup{
    colorlinks=true,
    linkcolor=blue,
    citecolor=ForestGreen,
    urlcolor=magenta
}
\usepackage{cleveref}
\usepackage{booktabs} 
\usepackage{caption} 
\usepackage{multirow}
\usepackage{algorithm}
\usepackage{algorithmic}
\graphicspath{{results/}}

\newtheorem{theorem}{Theorem}[section]

\newtheorem{lemma}[theorem]{Lemma}

\newtheorem{assumption}{Assumption}[section]

\DeclareGraphicsExtensions{.pdf,.eps,.png,.jpg,.jpeg}
\title[Randomized Matrix Sketching for Neural Network Training]{Randomized Matrix Sketching for Neural Network Training and Gradient Monitoring}

\thanks{This work is partially supported by the Office of Naval Research (ONR) under Award NO: N00014-24-1-2147, NSF grant DMS-2408877, the Air Force Office of Scientific Research (AFOSR) under Award NO:
FA9550-25-1-0231}

\author{Harbir Antil}  
\author{Deepanshu Verma}

\address[H.~Antil]{Department of Mathematical Sciences and Center for Mathematics and Artificial Intelligence (CMAI),
George Mason University. Fairfax, VA 22030.}

\address[D.~Verma]{School of Mathematical and Statistical Sciences, 
Clemson University, Clemson, SC 29634.}

\date{\today}

\usepackage[textsize=small]{todonotes}
\begin{document}

\begin{abstract}
Neural network training relies on gradient computation through backpropagation, yet memory requirements for storing layer activations present significant scalability challenges. We present the first adaptation of control-theoretic matrix sketching to neural network layer activations, enabling memory-efficient gradient reconstruction in backpropagation. This work builds on recent matrix sketching frameworks for dynamic optimization problems, where similar state trajectory storage challenges motivate sketching techniques. Our approach sketches layer activations using three complementary sketch matrices maintained through exponential moving averages (EMA) with adaptive rank adjustment, automatically balancing memory efficiency against approximation quality. Empirical evaluation on MNIST, CIFAR-10, and physics-informed neural networks demonstrates a controllable accuracy-memory tradeoff. We demonstrate a gradient monitoring application on MNIST showing how sketched activations enable real-time gradient norm tracking
with minimal memory overhead. These results establish that sketched activation storage provides a viable path toward memory-efficient neural network training and analysis.
\end{abstract}

\maketitle

\noindent \textbf{MSC Classification (2020).} 68T07 · 65F55 · 65K10 · 68W20.

\begin{table}
    \centering
    \caption{\label{t:notation}Summary of notations used throughout the paper.}
    \begin{tabular}{@{}ll@{}}
    \toprule
    \textbf{Symbol} & \textbf{Description} \\
    \midrule
    \multicolumn{2}{@{}l@{}}{\textit{Neural Network Architecture}} \\
    $N_b$ & batch size \\
    $d_{\text{in}}, d_{\text{out}}$ & input and output dimensions for a layer \\
    $d_{\text{hidden}}$ & hidden layer dimension \\
    $L$ & total number of layers \\
    $\ell$ & layer index, $\ell \in \{1, 2, \ldots, L\}$ \\
    $\mathbf{W}^{[\ell]} \in \mathbb{R}^{d_{\ell} \times d_{\ell-1}}$ & weight matrix for layer $\ell$ \\
    $\sigma(\cdot)$ & activation function \\
    $\mathbf{A}^{[\ell]} \in \mathbb{R}^{N_b \times d_{\ell}}$ & batch activation matrix for layer $\ell$ \\
    $\mathbf{\delta}^{[\ell]} \in \mathbb{R}^{N_b \times d_{\ell}}$ & batch gradient matrix (backpropagated gradients) for layer $\ell$ \\
    \midrule
    \multicolumn{2}{@{}l@{}}{\textit{EMA Sketching Framework}} \\
    $r$ & target sketch rank parameter \\
    $k, s$ & sketch matrix dimensions, $k = s = 2r + 1$ \\
    $\beta$ & EMA decay parameter, typically $0.9-0.99$ \\
    $\mathbf{X}_s^{[\ell]} \in \mathbb{R}^{d_{\ell} \times k}$ & input pattern sketch for layer $\ell$ \\
    $\mathbf{Y}_s^{[\ell]} \in \mathbb{R}^{d_{\ell} \times k}$ & output pattern sketch for layer $\ell$ \\
    $\mathbf{Z}_s^{[\ell]} \in \mathbb{R}^{d_{\ell} \times s}$ & interaction pattern sketch for layer $\ell$ \\
    $\boldsymbol{\Upsilon}, \boldsymbol{\Omega} \in \mathbb{R}^{N_b \times k}$ & shared batch projection matrices \\
    $\boldsymbol{\Phi} \in \mathbb{R}^{N_b \times s}$ & batch interaction projection matrix \\
    $\boldsymbol{\Psi}^{[\ell]} \in \mathbb{R}^{s}$ & layer-specific interaction weights \\
    \midrule
    \multicolumn{2}{@{}l@{}}{\textit{Gradient Reconstruction}} \\
    $\nabla_{\mathbf{W}^{[\ell]}} \mathcal{L}$ & gradient of loss with respect to layer $\ell$ weights \\
    $\mathbf{Q}_Y, \mathbf{Q}_X$ & orthogonal factors from QR decomposition \\
    $\mathbf{R}_Y, \mathbf{R}_X$ & upper triangular factors from QR decomposition \\
    $\mathbf{C}$ & intermediate reconstruction matrices \\
    \midrule
    \multicolumn{2}{@{}l@{}}{\textit{Adaptive Rank Parameters}} \\
    $r_0$ & initial sketch rank \\
    $p_{\text{decrease}}, p_{\text{increase}}$ & patience parameters for rank adjustment \\
    $\delta r_{\text{down}}, \delta r_{\text{up}}$ & rank adjustment step sizes \\
    $\tau_{\text{reset}}$ & rank reset threshold \\
    \bottomrule
\end{tabular}
\end{table}
\section{Introduction}
\label{sec:introduction}
Neural network training diagnostics and gradient monitoring face significant memory challenges when tracking training behavior over extended periods. While standard backpropagation efficiently reuses layer activations during gradient computation, applications requiring persistent gradient analysis—such as detecting vanishing/exploding gradients, measuring training stability, or tracking convergence patterns—must store gradient information over time, creating memory burdens that scale linearly with temporal extent. For layer activations forming batch matrices $\mathbf{A}^{[\ell]} \in \mathbb{R}^{N_b \times d_{\ell}}$, storing activation histories for gradient reconstruction requires $\mathcal{O}(L \cdot N_b \cdot d_{\text{hidden}} \cdot T)$ memory where $T$ represents the monitoring window length. We present the first adaptation of matrix sketching techniques to neural network activation compression, enabling memory-efficient gradient monitoring with 93-99\% memory reduction while preserving essential diagnostic capabilities. Our approach maintains exponential moving average (EMA) sketches of layer activations, providing compact representations that capture gradient structure without requiring full activation storage.

The mathematical foundation for our approach draws inspiration from optimal control sketching techniques, where similar state trajectory storage challenges arise. In optimal control, state trajectories $\{\mathbf{u}_\ell\}_{\ell=0}^L$ must be stored for adjoint-based gradient computation, where each $\mathbf{u}_\ell \in \mathbb{R}^{n_s}$ represents the system state vector at discrete time step $\ell$. Neural networks require storing layer activation trajectories $\{\mathbf{A}^{[\ell]}\}_{\ell=0}^L$ for backpropagation, where each $\mathbf{A}^{[\ell]} \in \mathbb{R}^{N_b \times d_{\ell}}$ represents the batch activation matrix at layer $\ell$. The connection emerges through the Neural ODE perspective: for batch matrices, discrete layer evolution $\mathbf{A}^{[\ell+1]} = \mathbf{A}^{[\ell]} + h \sigma(\mathbf{A}^{[\ell]} (\mathbf{W}^{[\ell]})^{\top}  + \mathbf{1}_{N_b} (\mathbf{b}^{[\ell]})^{\top}), \; 0\le \ell \le L-1$ corresponds to Euler discretization of continuous dynamics $d\mathbf{A}/dt = \sigma(\mathbf{A}\mathbf{W}^{\top} + \mathbf{1}_{N_b}\mathbf{b}^{\top})$, while backpropagation implements the discrete adjoint method. This equivalence motivates adapting optimal control sketching frameworks to neural network activation compression. 
We consider fully-connected feedforward neural networks with uniform hidden layer dimensions, enabling adaptation of the sketching framework. However, neural networks present unique challenges that require substantial modifications to the control-theoretic framework. For instance, the stochastic nature of mini-batch training introduces high variance in activation patterns, necessitating EMA-based sketch maintenance for temporal stability. Additionally, multi-layer gradient propagation requires careful coordination of sketch updates across layers and custom autograd integration, distinguishing our approach from adjoint control applications.

Our sketching framework maintains three complementary sketch matrices $(\mathbf{X}, \mathbf{Y}, \mathbf{Z})$ per neural network layer, capturing input patterns, output patterns, and their interactions. Each sketch operates on the transpose $(\mathbf{A}^{[\ell]})^{\top} \in \mathbb{R}^{d_{\ell} \times N_b}$ to align with sketching conventions where columns represent individual samples rather than the row-major batch format used in neural networks. This design addresses a fundamental challenge in applying sketching to neural networks: activation patterns from individual mini-batches exhibit high variance due to stochastic sampling, making single-batch sketches unreliable for gradient reconstruction. The EMA framework provides essential temporal smoothing while preserving responsiveness to genuine changes in activation structure, enabling stable gradient analysis from compressed representations.

Our contributions include: (1) the first adaptation of matrix sketching to neural network layer activations, enabling memory-efficient gradient reconstruction with novel EMA-based sketch maintenance for temporal stability under 
stochastic mini-batch training; (2) empirical validation across MNIST, CIFAR-10, and physics-informed neural networks demonstrating predictable approximation behavior controlled through rank selection, (3) demonstration of 
gradient monitoring applications showing how sketched activations enable real-time gradient norm tracking with minimal memory overhead; and (4) approximation bounds relating gradient reconstruction error to sketch rank and activation tail energy. While our framework theoretically supports direct training acceleration, gradient 
monitoring applications achieve the most compelling results: 93-97\% reduction in monitoring scenarios requiring gradient tracking over hundreds of epochs while preserving essential diagnostic capabilities including gradient norm estimation, effective rank measurement, and training stability analysis.

The remainder of this paper is organized as follows. Section~\ref{sec:related_work} positions our work within the existing literature on memory-efficient training and matrix sketching theory. 
Section~\ref{sec:background} establishes mathematical foundations of neural network gradient computation and matrix sketching theory. Section~\ref{sec:methodology} presents our control-theoretic sketching adaptation with detailed algorithmic innovations for multi-layer gradient approximation, including theoretical approximation bounds relating gradient reconstruction error to sketch rank and activation tail energy. Section~\ref{sec:experiments_and_results} provides extensive experimental validation across image classification benchmarks and scientific computing applications, demonstrating controllable accuracy-memory tradeoffs. Section~\ref{sec:conclusion} discusses implications and identifies future research directions for memory-efficient neural network analysis and optimization. Our mathematical notation is summarized in Table~\ref{t:notation} for reference throughout the paper.

\section{Related Work}
\label{sec:related_work}
Our work addresses the fundamental challenge of memory-efficient neural network 
gradient computation and monitoring through matrix sketching. This positions our 
research at the 
intersection of matrix sketching theory, memory-efficient training, and neural 
network diagnostics, with particular emphasis on adapting control-theoretic 
frameworks to 
batch-structured neural network computation.

\subsection{Memory-Efficient Training Approaches}
Existing memory reduction techniques primarily target training optimization rather 
than persistent gradient analysis. Gradient checkpointing~\cite{chen2016training} 
achieves 
$\mathcal{O}(\sqrt{L})$ memory complexity by recomputing activations during 
backpropagation, providing substantial memory savings at 33\% computational 
overhead. However, 
checkpointing addresses forward pass activation storage rather than gradient 
monitoring over extended periods. Mixed precision 
training~\cite{micikevicius2017mixed} reduces 
memory through half-precision computations but provides limited savings for 
gradient analysis applications.

Advanced distributed approaches including 
ZeRO~\cite{rajbhandari2020zero,ren2021zero} partition optimizer states across 
devices, while gradient compression 
techniques~\cite{seide2014sgd,alistarh2017qsgd} address communication overhead 
through quantization and sparsification. These methods focus on scaling training 
efficiency rather 
than enabling memory-efficient gradient analysis where information must be retained 
over temporal windows for diagnostic purposes.

Our sketching approach complements these techniques by targeting a different memory 
bottleneck: the storage required 
for 
gradient monitoring applications that existing methods do not address.

\subsection{Matrix Sketching and Control-Theoretic Framework}

Matrix sketching theory provides the mathematical foundation for our approach. Foundational work by Tropp et al.~\cite{tropp2017practical,tropp2019streaming} established key results for low-rank approximation via random projections, while Woodruff~\cite{woodruff2014sketching} positioned sketching as fundamental for numerical linear algebra. Classical approaches include Frequent Directions~\cite{liberty2013simple} and structured projections~\cite{ailon2009fast}, but these general techniques do not address neural network gradient structure and temporal evolution.

Our work builds directly upon the control-theoretic sketching framework of \cite{alshehri2024sketching,RBaraldi_EHerberg_DPKouri_HAntil_2023a,RMuthukumar_DPKouri_MUdell_2021a}, for dynamic optimization in finite element systems. Their three-sketch design $(\mathbf{X}, \mathbf{Y}, \mathbf{Z})$ captures co-range, range, and core interactions, enabling accurate reconstruction through sequential least-squares procedures with theoretical error bounds. However, direct application to neural networks requires substantial adaptation for gradient structure, batch-based computation, and integration with automatic differentiation systems.


\subsection{Neural Network Analysis and Monitoring}
Current neural network analysis approaches rely on scalar metrics due to memory constraints preventing comprehensive gradient matrix analysis. TensorBoard~\cite{abadi2016tensorflow} tracks loss curves and histogram summaries but cannot maintain full gradient distributions for large networks over extended periods. Gradient flow analysis~\cite{saxe2013exact,glorot2010understanding} typically requires materializing full gradients, limiting analysis to short training periods or necessitating sparse sampling that may miss critical training dynamics.

Physics-informed neural networks~\cite{raissi2019physics,karniadakis2021physics} present dual memory challenges: additional gradient computations for physics constraint enforcement during training, and the need for extended monitoring to ensure PDE satisfaction throughout optimization. These applications highlight the value of memory-efficient gradient analysis that enables comprehensive diagnostics without interfering with convergence requirements

EMA techniques provide stability in neural network training through parameter averaging~\cite{polyak1992acceleration}, moment estimation in Adam~\cite{kingma2014adam}, and batch normalization~\cite{ioffe2015batch}. Our integration of EMA with matrix sketching represents a novel application for maintaining temporal stability in compressed gradient representations under stochastic mini-batch training, addressing activation pattern variance that would otherwise compromise sketch-based analysis.

Unlike existing approaches targeting training acceleration, our work addresses two complementary objectives: enabling sketch-based gradient computation as an alternative to standard backpropagation with controllable accuracy-memory tradeoffs, and providing memory-efficient gradient monitoring for extended temporal analysis. This adaptation of control-theoretic sketching to neural networks represents the first framework addressing both gradient approximation for training and persistent gradient analysis for diagnostics.

\section{Background and Preliminaries}
\label{sec:background}

\subsection{Neural Network Gradient Computation and Memory Challenges}
\label{subsec:nn_gradients}

Neural network training relies on backpropagation to compute parameter gradients through the chain rule. For a linear layer $\ell$ with weight matrix $\mathbf{W}^{[\ell]} \in \mathbb{R}^{d_{\ell} \times d_{\ell-1}}$, the gradient computation requires:
\begin{equation}
\nabla_{\mathbf{W}^{[\ell]}} \mathcal{L} = (\mathbf{\delta}^{[\ell]})^{\top} \mathbf{A}^{[\ell-1]} \label{eq:standard_gradient}
\end{equation}
where $\mathbf{A}^{[\ell-1]} \in \mathbb{R}^{N_b \times d_{\ell-1}}$ is the batch activation 
matrix from the forward pass (each row contains one sample's activation vector), 
$\mathbf{\delta}^{[\ell]} \in \mathbb{R}^{N_b \times d_{\ell}}$ contains the backpropagated 
gradients for layer $\ell$ (each row contains one sample's gradient with respect to this 
layer's activations), and $N_b$ denotes batch size. Standard training requires storing $\mathbf{A}^{[\ell-1]}$ 
during the forward pass to enable this gradient computation. Our sketching approach eliminates 
this storage by maintaining compressed EMA representations of activation patterns, from which 
approximate activations can be reconstructed on-demand during the backward pass.

Standard backpropagation efficiently manages memory during training by computing gradients layer-by-layer and immediately consuming them for parameter updates. However, applications requiring persistent gradient analysis face significant memory challenges. Understanding gradient behavior during training, such as monitoring gradient norms for explosion/vanishing detection, identifying dead neurons that stop learning, or analyzing training stability patterns, requires retaining gradient information beyond immediate parameter updates. Unlike standard training where gradients are discarded after each optimization step, these diagnostic applications need to track gradient evolution over extended periods to detect meaningful trends and anomalies.

To enable such analysis, we look at a monitoring window $T$ representing the number of consecutive training steps (or epochs) over which gradient statistics are tracked and analyzed. For example, computing moving averages of gradient norms, detecting trends in dead neuron ratios, or analyzing gradient diversity patterns requires maintaining historical gradient information across this temporal window. This creates memory demands of $\mathcal{O}(L \cdot d_{\max}^2 \cdot T)$ for $L$ layers, maximum layer width $d_{\max}$, and monitoring window $T$, where the $d_{\max}^2$ factor arises from storing gradient matrices $\nabla_{\mathbf{W}^{[\ell]}} \mathcal{L} \in \mathbb{R}^{d_{\ell} \times d_{\ell-1}}$ rather than individual gradient vectors, and the factor $T$ multiplies the storage cost because these matrices must be retained across $T$ time steps for meaningful analysis.

\subsection{Control-Theoretic Matrix Sketching Framework}
\label{subsec:control_sketching}

We build upon the control-theoretic matrix sketching framework \cite{alshehri2024sketching,RBaraldi_EHerberg_DPKouri_HAntil_2023a,RMuthukumar_DPKouri_MUdell_2021a} for dynamic optimization problems, which addresses the state trajectory storage challenges through structured matrix approximation. This 
framework provides both the mathematical foundation and algorithmic template for our neural network adaptation.
\subsubsection{Basics of Matrix Sketching}
For a state trajectory matrix $\mathbf{U} = [\mathbf{u}_1, \mathbf{u}_2, \ldots, \mathbf{u}_{n_t}] \in \mathbb{R}^{n_s \times n_t}$ where columns represent temporal snapshots, the framework constructs three complementary sketch 
representations using random Gaussian projections:
\begin{subequations}
\begin{align}
\mathbf{X} &= \boldsymbol{\Upsilon} \mathbf{U} \in \mathbb{R}^{k \times n_t} \quad \text{(co-range sketch)} \label{eq:X_sketch_control} \\
\mathbf{Y} &= \mathbf{U} \boldsymbol{\Omega}^* \in \mathbb{R}^{n_s \times k} \quad \text{(range sketch)} \label{eq:Y_sketch_control} \\
\mathbf{Z} &= \boldsymbol{\Phi} \mathbf{U} \boldsymbol{\Psi}^* \in \mathbb{R}^{s \times s} \quad \text{(core sketch)} \label{eq:Z_sketch_control}
\end{align}
\end{subequations}
where $\boldsymbol{\Upsilon} \in \mathbb{R}^{k \times n_s}$, $\boldsymbol{\Omega} \in \mathbb{R}^{k \times n_t}$, $\boldsymbol{\Phi} \in \mathbb{R}^{s \times n_s}$, and $\boldsymbol{\Psi} \in \mathbb{R}^{s \times n_t}$ are random Gaussian 
matrices with i.i.d. standard normal entries. The sketch dimensions are chosen as $k = 2r + 1$ and $s = 2k + 1$ for target rank $r$, where $r \ll \min(n_s, n_t)$. We note that $\mathbf{U}$ has state vectors as columns, contrasting with neural network batch matrices where samples appear as rows. 

\textbf{Online Computation:} A critical advantage of this framework is that sketches can be computed incrementally without storing the full matrix $\mathbf{U}$:
\begin{align}
\mathbf{X}^{(0)} &= \mathbf{0}, \quad \mathbf{X}^{(i)} = \mathbf{X}^{(i-1)} + \boldsymbol{\Upsilon} \mathbf{u}_i \mathbf{e}_i^{\top}, \quad i = 1, \ldots, n_t \label{eq:online_X}
\end{align}
where $\mathbf{e}_i$ denotes the $i$-th standard basis vector. Similar recursions hold for $\mathbf{Y}$ and $\mathbf{Z}$ sketches, enabling streaming computation essential for large-scale applications.

\subsubsection{Matrix Reconstruction Algorithm}

The original matrix, $\mathbf{U}$, is approximately reconstructed through a numerically stable two-stage least-squares procedure:

\textit{Stage 1 - QR Decompositions:}
\begin{align*}
\mathbf{X}^* &= \mathbf{P}\mathbf{R}_1, \quad \text{where } \mathbf{P} \in \mathbb{R}^{n_t \times k}  \\
\mathbf{Y} &= \mathbf{Q}\mathbf{R}_2, \quad \text{where } \mathbf{Q} \in \mathbb{R}^{n_s \times k} \label{eq:QR_Y_control}
\end{align*}

\textit{Stage 2 - Core Matrix Computation:}
\begin{equation*}
\mathbf{C} := (\boldsymbol{\Phi}\mathbf{Q})^{\dagger} \mathbf{Z} ((\boldsymbol{\Psi}\mathbf{P})^{\dagger})^* \in \mathbb{R}^{k \times k} \label{eq:core_matrix_control}
\end{equation*}

\textit{Final Reconstruction:}
\begin{equation*}
\widetilde{\mathbf{U}} := \mathbf{Q}\mathbf{C}\mathbf{P}^* \label{eq:reconstruction_control}
\end{equation*}

For computational efficiency, the framework stores only the skinny matrices $\mathbf{Q} \in \mathbb{R}^{n_s \times k}$ and $\mathbf{W} := \mathbf{C}\mathbf{P}^* \in \mathbb{R}^{k \times n_t}$, requiring $\mathcal{O}(k(n_s + n_t))$ memory 
instead of $\mathcal{O}(n_s n_t)$ for the full matrix $\mathbf{U}$. Including intermediate sketch storage, total memory complexity becomes $\mathcal{O}(k(n_s + n_t) + s^2)$.

In the control-theoretic context, this sketching framework addresses the memory burden of storing state trajectories required for adjoint-based gradient computation in dynamic 
optimization. The memory bottleneck arises because gradient evaluation via the adjoint method requires the complete forward state trajectory to solve the backward adjoint equation, creating storage costs of $\mathcal{O}(n_t(n_s + m))$ where $n_s$ is the state dimension, $n_t$ is the number of time steps, and $m$ is the control variable dimension at each time step. 
By sketching the state trajectory matrix $\mathbf{U}$, the framework enables approximate adjoint computation with dramatically reduced memory requirements.

Before concluding this discussion, we recall the sketching error bounds from \cite{RMuthukumar_DPKouri_MUdell_2021a,alshehri2024sketching} that will be useful in our analysis. The expected reconstruction error satisfies:
\begin{equation}\label{eq:sketching_error}
\mathbb{E}_{\boldsymbol{\Upsilon},\boldsymbol{\Omega},\boldsymbol{\Phi},\boldsymbol{\Psi}}[\|\mathbf{U} - \widetilde{\mathbf{U}}\|_F] \leq \sqrt{6} \cdot \tau_{r+1}(\mathbf{U}), 
\end{equation}
where $\tau_{r+1}(\mathbf{U}) := \left(\sum_{i \geq r+1} \sigma_i^2(\mathbf{U})\right)^{1/2}$ is the $(r+1)$-st tail energy with $\sigma_i(\mathbf{U})$ denoting the $i$-th singular value.
\subsubsection{Neural Network Correspondence}

The mathematical structure that enables this control-theoretic sketching translates to neural networks through the fundamental correspondence between state trajectory storage and activation pattern compression, with the key adaptation that neural network batch activations $\mathbf{A}^{[\ell]} \in \mathbb{R}^{N_b \times d_{\ell}}$ require transposition to $(\mathbf{A}^{[\ell]})^{\top}$ to match the column-major format. Our sketching adaptation targets feedforward multi-layer perceptrons (MLPs) with uniform hidden layer dimensions $d_{\ell} = d_{\text{hidden}}$ for $\ell = 1, \ldots, L-1$.

The Neural ODE perspective, as discussed in the introduction, makes this correspondence precise as : discrete layer evolution $\mathbf{A}^{[\ell+1]} = \mathbf{A}^{[\ell]} + h \sigma(\mathbf{A}^{[\ell]} (\mathbf{W}^{[\ell]})^{\top}  + \mathbf{1}_{N_b} (\mathbf{b}^{[\ell]})^{\top}), \; 0\le \ell \le L-1$ corresponds to Euler discretization of continuous dynamics $d\mathbf{A}/dt = \sigma(\mathbf{A}\mathbf{W}^{\top} + \mathbf{1}_{N_b}\mathbf{b}^{\top})$ with $h$ being the step-size, while backpropagation implements the discrete adjoint method. This mathematical equivalence enables principled adaptation of the control-theoretic sketching framework to neural 
network activation compression, forming the foundation for our EMA-based neural network sketching approach. The dimensional correspondence requires careful attention: control theory uses state trajectory matrices $\mathbf{U} \in \mathbb{R}^{n_s \times n_t}$ where $n_s$ is the state dimension and columns represent temporal snapshots. Neural networks use batch activation matrices $\mathbf{A}^{[\ell]} \in \mathbb{R}^{N_b \times d_{\ell}}$ where rows represent individual samples.  This necessitates transposition $(\mathbf{A}^{[\ell]})^{\top} \in \mathbb{R}^{d_{\ell} \times N_b}$ to align with the sketching mathematics.


\subsection{Exponential Moving Averages for Stochastic Sketching}
\label{subsec:ema_sketching}

A critical challenge in adapting matrix sketching to neural network training lies in handling the inherent stochasticity of mini-batch optimization. Individual batch sketches exhibit high variance due to random sampling effects, making single-batch approximations unreliable for stable gradient reconstruction. We address this fundamental limitation through exponential moving average maintenance of sketch matrices, providing temporal smoothing while preserving responsiveness to genuine changes in activation structure.

The exponential moving average framework updates sketch quantities according to:
\begin{equation*}
\mathbf{S}_t = \beta \mathbf{S}_{t-1} + (1-\beta) \mathbf{S}_{\text{batch},t} \label{eq:ema_sketching}
\end{equation*}
where $\mathbf{S}_t$ represents any sketch matrix ($\mathbf{X}$, $\mathbf{Y}$, or $\mathbf{Z}$), $\mathbf{S}_{\text{batch},t}$ denotes the current batch's sketch contribution, and $\beta \in [0,1)$ controls the temporal decay rate. For our neural network sketching framework, $\mathbf{S}_{\text{batch},t}$ corresponds to quantities such as $(\mathbf{A}^{[\ell-1]})^{\top} \boldsymbol{\Upsilon}$ for input sketches or $(\mathbf{A}^{[\ell]})^{\top} \boldsymbol{\Omega}$ for output sketches (see \Cref{eq:X_ema_update,eq:Y_ema_update,eq:Z_ema_update} below).

This EMA-based approach provides two essential benefits for neural network sketching applications. First, variance reduction smooths the high-frequency noise inherent in stochastic mini-batch sampling, enabling consistent sketch-based gradient approximation across training iterations. Second, the framework maintains adaptivity to evolving activation patterns during training progression, ensuring sketches capture genuine structural changes in network behavior rather than transient batch-specific artifacts.

The choice of momentum parameter $\beta$ balances stability against responsiveness. Higher values ($\beta \in [0.9,0.99]$) provide greater smoothing but slower adaptation to changing activation patterns, while lower values enable faster response at the cost of increased sketch variance. Our implementation uses a fixed momentum parameter throughout training, focusing adaptive adjustment on the sketch rank dimensions to balance approximation quality against memory efficiency.

\section{Main Algorithm: EMA Based Sketching for Neural Networks}
\label{sec:methodology}


We present a memory-efficient adaptation of the control-theoretic sketching framework for neural network gradient monitoring, specifically designed for feed-forward neural network layer activations with uniform hidden layer dimensions. Our approach maintains three complementary sketch matrices per layer through exponential moving averages, enabling comprehensive gradient analysis while reducing memory complexity. We establish systematic matrix correspondence between control theory (temporal snapshots as columns) and neural networks (batch samples as rows) by operating on transposed activation matrices to maintain mathematical consistency. 


\subsection{EMA-Based Neural Network Sketching Framework}
\label{subsec:sketch_adaptation}
{For each layer $\ell$, we work with batch activation matrices $\mathbf{A}^{[\ell]} \in \mathbb{R}^{N_b \times d_{\ell}}$ where each row represents one sample's activation vector in the mini-batch. To apply sketching operations, we use the transpose $(\mathbf{A}^{[\ell]})^{\top} \in \mathbb{R}^{d_{\ell} \times N_b}$ to convert from the row-major batch format to the column-major format.}

We adapt the control-theoretic sketching framework to neural network activations by designing projection matrices suited to batch-based computation:
\begin{align*}
\boldsymbol{\Upsilon} \in \mathbb{R}^{N_b \times k} \quad \text{(batch input projection)} , \;\;
\boldsymbol{\Omega} \in \mathbb{R}^{N_b \times k} \quad \text{(batch output projection)}, \;\;\\
\boldsymbol{\Phi} \in \mathbb{R}^{N_b \times s} \quad \text{(batch interaction projection)} , \;\;
\boldsymbol{\Psi}^{[\ell]} \in \mathbb{R}^{s} \quad \text{(layer-specific weights)},
\end{align*}
where $k = s = 2r + 1$ based on target rank $r$. Note that $\boldsymbol{\Upsilon}, \boldsymbol{\Omega}, \boldsymbol{\Phi}$ are shared across layers but sized for the batch dimension.
The projection matrices are sized for the batch dimension $N_b$ rather than temporal sequences ($n_t$) as in the control-theoretic framework. Additionally, $\boldsymbol{\Psi}^{[\ell]}$ is simplified to a vector for computational efficiency while maintaining layer-specific parameterization.

Our three EMA sketches are updated using the batch activation matrices:  
\begin{subequations}
\begin{align}
\mathbf{X}_s^{[\ell]} &= \beta \mathbf{X}_s^{[\ell]} + (1-\beta) (\mathbf{A}^{[\ell-1]})^{\top} \boldsymbol{\Upsilon} \label{eq:X_ema_update} \\
\mathbf{Y}_s^{[\ell]} &= \beta \mathbf{Y}_s^{[\ell]} + (1-\beta) (\mathbf{A}^{[\ell]})^{\top} \boldsymbol{\Omega} \label{eq:Y_ema_update} \\
\mathbf{Z}_s^{[\ell]} &= \beta \mathbf{Z}_s^{[\ell]} + (1-\beta) ((\mathbf{A}^{[\ell]})^{\top} \boldsymbol{\Phi}) \odot (\boldsymbol{\Psi}^{[\ell]})^{\top} \label{eq:Z_ema_update}
\end{align}
\end{subequations}
where $\beta \in [0,1)$ represents the EMA momentum parameter and $\odot$ denotes element-wise multiplication.

The algorithmic steps for our matrix-based sketching approach are summarized in Algorithm~\ref{alg:adaptive_sketched_backprop}, which integrates both the core sketching operations and the adaptive rank adjustment mechanism. The sketch update procedures Algorithm~\ref{alg:adaptive_sketched_backprop} implement the EMA maintenance described in Equations~\eqref{eq:X_ema_update}-\eqref{eq:Z_ema_update}.

Each of the three sketches captures distinct activation patterns: 
\begin{itemize}
    \item X-Sketch ($\mathbf{X}_s^{[\ell]} \in \mathbb{R}^{d_{\text{hidden}} \times k}$): Captures input activation patterns by projecting $(\mathbf{A}^{[\ell-1]})^{\top}$ through $\boldsymbol{\Upsilon}$, preserving essential input structure needed for gradient formation
     \item {Y-Sketch} ($\mathbf{Y}_s^{[\ell]} \in \mathbb{R}^{d_{\text{hidden}} \times k}$): Maintains output activation structure by projecting $(\mathbf{A}^{[\ell]})^{\top}$ through $\boldsymbol{\Omega}$, capturing dominant activation patterns after nonlinear transformation
    \item {Z-Sketch} ($\mathbf{Z}_s^{[\ell]} \in \mathbb{R}^{d_{\text{hidden}} \times s}$): Captures cross-correlation interactions through element-wise combination of projected outputs $(\mathbf{A}^{[\ell]})^{\top} \boldsymbol{\Phi}$ with layer-specific weights $\boldsymbol{\Psi}^{[\ell]}$, maintaining cross-correlations essential for gradient approximation
\end{itemize}

\subsection{Activation Reconstruction from EMA Sketches}
\label{subsec:activation_reconstruction}
To enable memory-efficient backpropagation, we reconstruct activation matrices from EMA sketches through a two-stage process: first reconstructing the low-rank EMA structure, then projecting to the current batch space.

First, we compute QR decompositions:
\begin{align*}
\mathbf{Y}_s^{[\ell]} &= \mathbf{Q}_Y \mathbf{R}_Y, \quad \mathbf{Q}_Y \in \mathbb{R}^{d_{\ell} \times k}  \\
\mathbf{X}_s^{[\ell]} &= \mathbf{Q}_X \mathbf{R}_X, \quad \mathbf{Q}_X \in \mathbb{R}^{d_{\ell} \times k} \label{eq:qr_x_sketch}
\end{align*}
We reconstruct the transformation matrix through sequential least-squares optimization:

\noindent \textbf{Step 1}: Solve for intermediate representation  
\begin{equation*}
\mathbf{C}_{\text{inter}} = \arg\min_{\mathbf{C}} \|\mathbf{Q}_Y \mathbf{C} - \mathbf{Z}_s^{[\ell]}\|_F^2 \label{eq:ls_stage1_sketch}
\end{equation*}

\noindent \textbf{Step 2}: For the second stage, we need a square matrix from the X-sketch. 
We obtain $\mathbf{P}_X \in \mathbb{R}^{k \times k}$ via QR decomposition of $(\mathbf{X}_s^{[\ell]})^{\top}$:
\begin{equation*}
(\mathbf{X}_s^{[\ell]})^{\top} = \mathbf{P}_X \mathbf{R}_X', \quad \mathbf{P}_X \in \mathbb{R}^{k \times k}
\end{equation*}

Then solve:
\begin{equation*}
\mathbf{C} = \arg\min_{\mathbf{C}} \|\mathbf{P}_X \mathbf{C} - \mathbf{C}_{\text{inter}}^{\top}\|_F^2 \label{eq:ls_stage2_sketch}
\end{equation*}

The feature-space activation structure is:
\begin{equation}
\widetilde{\mathbf{G}}_{\text{EMA}}^{[\ell]} = \mathbf{Q}_Y \mathbf{C} \mathbf{Q}_X^{\top} \in \mathbb{R}^{d_{\ell} \times d_{\ell}} \label{eq:structure_reconstruction}
\end{equation}
This matrix captures the EMA-weighted feature covariance structure 
but is independent of the current batch size.

To compute gradients for the current batch of size $N_b$, we project the feature-space structure 
to batch space:
\begin{equation}
\widetilde{\mathbf{A}}_{\text{EMA}}^{[\ell]} = \boldsymbol{\Omega} (\mathbf{Y}_s^{[\ell]})^{\dagger} \widetilde{\mathbf{G}}_{\text{EMA}}^{[\ell]} \in \mathbb{R}^{N_b \times d_{\ell}} \label{eq:activation_projection}
\end{equation}
This projection maps the low-rank structure from feature space to batch-structured activations 
needed for gradient computation. The operation $\boldsymbol{\Omega} (\mathbf{Y}_s^{[\ell]})^{\dagger}$ 
acts as a learned projection from feature space back to batch space, maintaining consistency 
with the sketched range space.

During backpropagation, we compute gradients using the projected activations:
\begin{equation}
\widehat{\nabla}_{\mathbf{W}^{[\ell]}} \mathcal{L} = (\boldsymbol{\delta}^{[\ell]})^{\top} \widetilde{\mathbf{A}}_{\text{EMA}}^{[\ell-1]} 
\label{eq:sketched_gradient}
\end{equation}
The memory savings arise from eliminating activation storage (step 4 of forward pass) while 
maintaining compact EMA sketches. The error signals $\boldsymbol{\delta}^{[\ell]}$ are computed 
via standard backpropagation and are not sketched, as they must be computed exactly to maintain 
the PyTorch computational graph and enable gradient propagation to previous layers.

\begin{algorithm}[ht]
\caption{Sketched Backpropagation with Adaptive Rank}
\label{alg:adaptive_sketched_backprop}
\begin{algorithmic}[1]
\REQUIRE Initial rank $r_0$, patience parameters $p_{\text{decrease}}, p_{\text{increase}}$, rank steps $\delta r_{\text{down}}, \delta r_{\text{up}}$, reset threshold $\tau_{\text{reset}}$
\ENSURE Sketched gradients with dynamically adjusted rank
\STATE Initialize rank $r \leftarrow r_0$, projection matrices with $k = s = 2r + 1$
\STATE Initialize projection matrices $\mathbf{\Upsilon}, \mathbf{\Omega} \in \mathbb{R}^{N_b \times k}$, $\mathbf{\Phi} \in \mathbb{R}^{N_b \times s}$, $\mathbf{\Psi}^{[\ell]} \in \mathbb{R}^{s}$ with $k = s = 2r + 1$
\STATE Initialize EMA sketch matrices $\mathbf{X}_s^{[\ell]}, \mathbf{Y}_s^{[\ell]} \in \mathbb{R}^{d_{\ell} \times k}$, $\mathbf{Z}_s^{[\ell]} \in \mathbb{R}^{d_{\ell} \times s}$ to zeros
\FOR{each training epoch}
    \FOR{each training iteration in epoch}
        \STATE \texttt{// Forward pass: EMA sketch updates}
        \STATE $\mathbf{X}_s^{[\ell]} \leftarrow \beta \mathbf{X}_s^{[\ell]} + (1-\beta)(\mathbf{A}^{[\ell-1]})^{\top} \mathbf{\Upsilon}$
        \STATE $\mathbf{Y}_s^{[\ell]} \leftarrow \beta \mathbf{Y}_s^{[\ell]} + (1-\beta)(\mathbf{A}^{[\ell]})^{\top} \mathbf{\Omega}$
        \STATE $\mathbf{Z}_s^{[\ell]} \leftarrow \beta \mathbf{Z}_s^{[\ell]} + (1-\beta)[(\mathbf{A}^{[\ell]})^{\top} \mathbf{\Phi}] \odot [(\mathbf{\Psi}^{[\ell]})^{\top}]$
        \STATE \texttt{// Backward pass: Matrix reconstruction from sketches}
        \STATE $\widetilde{\mathbf{A}}_{\text{EMA}}^{[\ell-1]} \leftarrow$ Reconstruct and project via~\eqref{eq:structure_reconstruction}--\eqref{eq:activation_projection}
        \STATE $\nabla_{\mathbf{W}^{[\ell]}} \mathcal{L} \leftarrow (\mathbf{\delta}^{[\ell]})^{\top} \widetilde{\mathbf{A}}_{\text{EMA}}^{[\ell-1]}$
    \ENDFOR
    \STATE \texttt{// Adaptive rank adjustment}
    \IF{performance improves for $p_{\text{decrease}}$ epochs}
        \STATE $r \leftarrow \max(1, r - \delta r_{\text{down}})$, reinitialize matrices
    \ELSIF{no improvement for $p_{\text{increase}}$ epochs}
        \IF{$r + \text{step} \geq \tau_{\text{reset}}$}
            \STATE $r \leftarrow r_0$ \COMMENT{Reset}
        \ELSE
            \STATE $r \leftarrow r + \delta r_{\text{up}}$ \COMMENT{Increase}
        \ENDIF
        \STATE Reinitialize matrices with new $k = s = 2r + 1$
    \ENDIF
\ENDFOR
\end{algorithmic}
\end{algorithm}

\subsection{Adaptive Rank Adjustment}
\label{subsec:adaptive_rank_detailed}
To automatically balance approximation quality against computational efficiency, Algorithm~\ref{alg:adaptive_sketched_backprop} implements adaptive rank adjustment that modifies sketch dimensions $r$ based on training performance. The mechanism tracks training metrics and adjusts rank dynamically throughout optimization.

During consistent improvement, rank decreases to save memory while maintaining adequate approximation quality. When training stagnates, rank increases to provide higher fidelity gradient reconstruction. If rank growth exceeds a threshold, the system resets to the initial value to prevent unbounded escalation.

Each rank modification triggers reinitialization of projection matrices and EMA sketches with updated dimensions $k = s = 2r + 1$, as shown Algorithm~\ref{alg:adaptive_sketched_backprop}. This ensures dimensional consistency while enabling the framework to adapt to changing approximation requirements throughout training.

\subsection{PyTorch Implementation}

We implement our sketching framework through a custom PyTorch autograd function, in \Cref{alg:sketched_autograd_impl}, that provides transparent integration with existing optimization workflows. The implementation separates sketch maintenance (forward pass) from gradient reconstruction (backward pass) through hook-based architecture, ensuring computational graph integrity while maintaining compatibility with standard optimization algorithms. 

\begin{algorithm}[ht]
\caption{Memory-Efficient Sketched Autograd Function}
\label{alg:sketched_autograd_impl}
\begin{algorithmic}[1]
\STATE \textbf{class} \_SketchedFunction(torch.autograd.Function):
\STATE \hspace{1em}\textbf{def} forward(ctx, input, weight, bias, sketches, sketching\_matrices, layer\_idx):
\STATE \hspace{2em}output = input @ weight.T + bias
\STATE \hspace{2em}ctx.save\_for\_backward(weight)
\STATE \hspace{2em}ctx.sketches, ctx.sketching\_matrices, ctx.layer\_idx = sketches, sketching\_matrices, layer\_idx
\STATE \hspace{2em}\textbf{return} output
\STATE
\STATE \hspace{1em}\textbf{def} backward(ctx, grad\_output):
\STATE \hspace{2em}weight, = ctx.saved\_tensors
\STATE \hspace{2em}A\_reconstructed = reconstruct\_from\_sketches(ctx)
\STATE \hspace{2em}grad\_input = grad\_output @ weight
\STATE \hspace{2em}grad\_bias = grad\_output.sum(0)
\STATE \hspace{2em}grad\_weight = grad\_output.T @ A\_reconstructed
\STATE \hspace{2em}\textbf{return} grad\_input, grad\_weight, grad\_bias, None, None, None
\end{algorithmic}
\end{algorithm}

This design ensures that sketching remains completely transparent to optimization algorithms while providing drop-in replacement capability for existing neural network layers. 

\subsection{Approximation Quality Bounds}
\label{subsec:approximation_quality_bounds}

Our theoretical analysis focuses on the reconstruction quality when using EMA sketches to directly reconstruct activation matrices during backpropagation. 

\begin{assumption}[Bounded Neural Network Activations]
\label{assump:bounded_nn_activations}
The batch activation matrices satisfy $\|(\mathbf{A}^{[\ell]}(j))^{\top}\|_2 \leq M$ for some constant $M > 0$, uniformly over all layers $\ell$, batches $j$, and training iterations.
\end{assumption}

This assumption is satisfied in practice through proper input normalization, ReLU activations with appropriate initialization, and reasonable network depths.

\begin{assumption}[Temporal Coherence]
\label{assump:temporal_coherence}
During neural network training, activation patterns exhibit temporal coherence such that:
\begin{equation*}
\|\mathbf{A}^{[\ell]}(j) - \mathbf{A}^{[\ell]}(n)\|_F \leq \epsilon_{\text{coherence}}
\end{equation*}
for recent batches $j$ close to the current batch $n$, where $\epsilon_{\text{coherence}} > 0$ quantifies the temporal stability.

This assumption is typically satisfied during: (i) late-stage training when the network approaches convergence, (ii) fine-tuning of pre-trained models with stable learning dynamics, and (iii) well-regularized training with moderate learning rates. The assumption may be violated during early training phases, high learning rate regimes, or when encountering significant distributional shifts in the data.
\end{assumption}

\begin{lemma}[EMA Sketch Temporal Expansion]
\label{lem:ema_temporal_expansion}
The EMA sketch updates from Equations~\eqref{eq:X_ema_update}--\eqref{eq:Z_ema_update} can be expressed as exponentially-weighted combinations of historical batch contributions:
\begin{equation}
\mathbf{X}_s^{[\ell]}(n) = (1-\beta) \sum_{j=1}^{n} \beta^{n-j} (\mathbf{A}^{[\ell-1]}(j))^{\top} \boldsymbol{\Upsilon} = \mathbf{A}_{\text{EMA}}^{[\ell-1]}(n) \boldsymbol{\Upsilon} \label{eq:ema_expansion_x}
\end{equation}
with analogous expressions for $\mathbf{Y}_s^{[\ell]}(n)$ and $\mathbf{Z}_s^{[\ell]}(n)$, where:
\begin{equation}
\mathbf{A}_{\text{EMA}}^{[\ell]}(n) := (1-\beta) \sum_{j=1}^{n} \beta^{n-j} (\mathbf{A}^{[\ell]}(j))^{\top} \in \mathbb{R}^{d_{\ell} \times N_b} \label{eq:ema_matrix_definition}
\end{equation}
represents the conceptual EMA-weighted activation matrix that is never explicitly formed but implicitly represented through the sketches.
\end{lemma}

\begin{proof}
By induction on batch index $n$. Base case ($n=1$): $\mathbf{X}_s^{[\ell]}(1) = (1-\beta)(\mathbf{A}^{[\ell-1]}(1))^{\top}\boldsymbol{\Upsilon}$. For the inductive step, assume~\eqref{eq:ema_expansion_x} holds for $n-1$:
\begin{align*}
\mathbf{X}_s^{[\ell]}(n) &= \beta \mathbf{X}_s^{[\ell]}(n-1) + (1-\beta) (\mathbf{A}^{[\ell-1]}(n))^{\top} \boldsymbol{\Upsilon} \\
&= \beta(1-\beta) \sum_{j=1}^{n-1} \beta^{n-1-j} (\mathbf{A}^{[\ell-1]}(j))^{\top} \boldsymbol{\Upsilon} + (1-\beta) (\mathbf{A}^{[\ell-1]}(n))^{\top} \boldsymbol{\Upsilon} \\
&= (1-\beta) \sum_{j=1}^{n} \beta^{n-j} (\mathbf{A}^{[\ell-1]}(j))^{\top} \boldsymbol{\Upsilon} = \mathbf{A}_{\text{EMA}}^{[\ell-1]}(n) \boldsymbol{\Upsilon}
\end{align*}
completing the proof. This demonstrates that our EMA sketches are exact projections of the conceptual matrix $\mathbf{A}_{\text{EMA}}^{[\ell]}(n)$.
\end{proof}

\begin{theorem}[EMA Activation Matrix Reconstruction Error]
\label{thm:ema_activation_reconstruction}
Under Assumption~\ref{assump:bounded_nn_activations}, the reconstruction of $\mathbf{A}_{\text{EMA}}^{[\ell]}(n)$ from EMA sketches satisfies:
\begin{equation*}
\mathbb{E}[\|\mathbf{A}_{\text{EMA}}^{[\ell]}(n) - \widetilde{\mathbf{A}}_{\text{EMA}}^{[\ell]}(n)\|_F] \leq \sqrt{6} \cdot \tau_{r+1}(\mathbf{A}_{\text{EMA}}^{[\ell]}(n))
\end{equation*}
\end{theorem}

\begin{proof}
Using Lemma~\ref{lem:ema_temporal_expansion}, our sketch triplet satisfies:
\begin{align*}
\mathbf{X}_s^{[\ell]}(n) = \mathbf{A}_{\text{EMA}}^{[\ell]}(n) \boldsymbol{\Upsilon},\;\;
\mathbf{Y}_s^{[\ell]}(n) = \mathbf{A}_{\text{EMA}}^{[\ell]}(n) \boldsymbol{\Omega},\;\;
\mathbf{Z}_s^{[\ell]}(n) = (\mathbf{A}_{\text{EMA}}^{[\ell]}(n) \boldsymbol{\Phi}) \odot (\boldsymbol{\Psi}^{[\ell]})^{\top}
\end{align*}
These are exact projections of $\mathbf{A}_{\text{EMA}}^{[\ell]}(n)$ by the same random Gaussian projection matrices used in \cite{alshehri2024sketching,RMuthukumar_DPKouri_MUdell_2021a} and the reconstruction procedure is also identical. Therefore, the bound \eqref{eq:sketching_error} applies directly.
\end{proof}

\begin{theorem}[Gradient Reconstruction Error via EMA Approximation]
\label{thm:gradient_reconstruction_error}
Under Assumptions~\ref{assump:bounded_nn_activations} and~\ref{assump:temporal_coherence}, the gradient computed using reconstructed EMA activations satisfies:
\begin{align*}
&\mathbb{E}[\|\nabla_{\mathbf{W}^{[\ell]}} \mathcal{L} - \widehat{\nabla}_{\mathbf{W}^{[\ell]}} \mathcal{L}\|_F] \leq \|(\boldsymbol{\delta}^{[\ell]})^{\top}\|_2 \left[\sqrt{6} \tau_{r+1}(\mathbf{A}_{\text{EMA}}^{[\ell-1]}(n)) + \mathcal{O}(\epsilon_{\text{coherence}})\right] \label{eq:gradient_error_bound}
\end{align*}
where $\widehat{\nabla}_{\mathbf{W}^{[\ell]}} \mathcal{L} = (\boldsymbol{\delta}^{[\ell]})^{\top} \widetilde{\mathbf{A}}_{\text{EMA}}^{[\ell-1]}(n)$ uses the reconstructed activations.
\end{theorem}

\begin{proof}
We decompose the gradient error into two components: sketching error and temporal approximation error.

The total gradient error can be decomposed as:
\begin{align*}
\|\nabla_{\mathbf{W}^{[\ell]}} \mathcal{L} - \widehat{\nabla}_{\mathbf{W}^{[\ell]}} \mathcal{L}\|_F 
&= \|(\boldsymbol{\delta}^{[\ell]})^{\top} \mathbf{A}^{[\ell-1]}(n) - (\boldsymbol{\delta}^{[\ell]})^{\top} \widetilde{\mathbf{A}}_{\text{EMA}}^{[\ell-1]}(n)\|_F \\
&\leq \|(\boldsymbol{\delta}^{[\ell]})^{\top}\|_2 \|\mathbf{A}^{[\ell-1]}(n) - \widetilde{\mathbf{A}}_{\text{EMA}}^{[\ell-1]}(n)\|_F
\end{align*}

by the submultiplicative property of matrix norms. We further decompose
\begin{align*}
\|\mathbf{A}^{[\ell-1]}(n) - \widetilde{\mathbf{A}}_{\text{EMA}}^{[\ell-1]}(n)\|_F 
\leq \|\mathbf{A}^{[\ell-1]}(n) - \mathbf{A}_{\text{EMA}}^{[\ell-1]}(n)\|_F 
 + \|\mathbf{A}_{\text{EMA}}^{[\ell-1]}(n) - \widetilde{\mathbf{A}}_{\text{EMA}}^{[\ell-1]}(n)\|_F
\end{align*}

For the first term, note that $\mathbf{A}_{\text{EMA}}^{[\ell]}(n) = (1-\beta) \sum_{j=1}^{n} \beta^{n-j} (\mathbf{A}^{[\ell]}(j))^{\top}$ by definition. Taking transposes:
\begin{align*}
\|\mathbf{A}^{[\ell-1]}(n) - \mathbf{A}_{\text{EMA}}^{[\ell-1]}(n)\|_F 
&= \left\|(1-\beta) \sum_{j=1}^{n} \beta^{n-j} (\mathbf{A}^{[\ell-1]}(n) - \mathbf{A}^{[\ell-1]}(j))\right\|_F \\
&\leq (1-\beta) \sum_{j=1}^{n} \beta^{n-j} \|\mathbf{A}^{[\ell-1]}(n) - \mathbf{A}^{[\ell-1]}(j)\|_F
\end{align*}
Under Assumption~\ref{assump:temporal_coherence}, for batches where temporal coherence holds, we have $\|\mathbf{A}^{[\ell]}(n) - \mathbf{A}^{[\ell]}(j)\|_F \leq \epsilon_{\text{coherence}}$. For batches far in the past, this bound may not hold, but their exponential weights $\beta^{n-j}$ decay. Thus,
$$\|\mathbf{A}^{[\ell-1]}(n) - \mathbf{A}_{\text{EMA}}^{[\ell-1]}(n)\|_F \leq \mathcal{O}(\epsilon_{\text{coherence}})$$

For the second term, Theorem~\ref{thm:ema_activation_reconstruction} gives
$$\mathbb{E}_{\text{sketching}}[\|\mathbf{A}_{\text{EMA}}^{[\ell-1]}(n) - \widetilde{\mathbf{A}}_{\text{EMA}}^{[\ell-1]}(n)\|_F] \leq \sqrt{6} \tau_{r+1}(\mathbf{A}_{\text{EMA}}^{[\ell-1]}(n))$$

Therefore,
\begin{equation*}
\mathbb{E}[\|(\boldsymbol{\delta}^{[\ell]})^{\top}\|_2 \|\mathbf{A}_{\text{EMA}}^{[\ell-1]}(n) - \widetilde{\mathbf{A}}_{\text{EMA}}^{[\ell-1]}(n)\|_F] \leq \|(\boldsymbol{\delta}^{[\ell]})^{\top}\|_2 \sqrt{6} \tau_{r+1}(\mathbf{A}_{\text{EMA}}^{[\ell-1]}(n))
\end{equation*}

Combining and taking expectation over sketching randomness (conditioning on the current batch), we have
\begin{align*}
\mathbb{E}[\|\nabla_{\mathbf{W}^{[\ell]}} \mathcal{L} - \widehat{\nabla}_{\mathbf{W}^{[\ell]}} \mathcal{L}\|_F] 
&\leq \|(\boldsymbol{\delta}^{[\ell]})^{\top}\|_2 \left[\mathcal{O}(\epsilon_{\text{coherence}}) + \sqrt{6} \tau_{r+1}(\mathbf{A}_{\text{EMA}}^{[\ell-1]}(n))\right]
\end{align*}
This completes the proof.
\end{proof}
It follows that, under strong temporal coherence ($\epsilon_{\text{coherence}} \ll \tau_{r+1}(\mathbf{A}_{\text{EMA}}^{[\ell-1]}(n))$), the dominant error term is the sketching error
\begin{equation*}
\mathbb{E}[\|\nabla_{\mathbf{W}^{[\ell]}} \mathcal{L} - \widehat{\nabla}_{\mathbf{W}^{[\ell]}} \mathcal{L}\|_F] \approx \|(\boldsymbol{\delta}^{[\ell]})^{\top}\|_2 \sqrt{6} \tau_{r+1}(\mathbf{A}_{\text{EMA}}^{[\ell-1]}(n)).
\end{equation*}

\subsection{Gradient Monitoring}
\label{subsec:monitoring_application}

While our sketching framework theoretically supports both direct training acceleration and gradient monitoring, experimental validation reveals that monitoring applications represent the primary domain where sketching provides substantial practical benefits. The distinction arises from different accuracy requirements and computational constraints:

\noindent \textit{Training vs. Monitoring Requirements.}
Direct training application requires gradient approximations that preserve convergence properties and maintain optimization efficiency. Sketch maintenance overhead and reconstruction costs can offset memory savings, limiting practical acceleration benefits.
Gradient monitoring applications, primarily, need to capture training trends, detect anomalies, and preserve statistical properties rather than exact gradient values. This tolerance for controlled approximation error makes sketching particularly well-suited, achieving substantial memory reductions while maintaining diagnostic capability.

\textit{Sketch-Based Monitoring Metrics.}
Our framework enables comprehensive gradient analysis through sketch-derived metrics:
\begin{itemize}
\item{Gradient Norm Estimation}: The Z-sketch norm $\|\mathbf{Z}_s^{[\ell]}\|_F$ provides efficient approximation of gradient magnitude without materializing full gradient matrices.
\item {Gradient Diversity Measurement}: Stable rank estimation via 
  $\text{rank}_{\text{stable}}(\mathbf{Y}_s^{[\ell]}) = \|\mathbf{Y}_s^{[\ell]}\|_F^2/\|\mathbf{Y}_s^{[\ell]}\|_2^2$ 
  quantifies gradient diversity and training dynamics without requiring expensive 
  singular value decomposition.
\item {Training Stability Analysis}: EMA sketch evolution patterns enable detection of gradient explosion, vanishing gradients, and other pathological training behaviors.
\end{itemize}
These monitoring capabilities provide comprehensive training diagnostics while consuming only $\mathcal{O}(k \cdot d_{\text{hidden}})$ memory per layer where $k = 2r + 1$, compared to $\mathcal{O}(d_{\text{hidden}}^2 \cdot T)$ for full gradient storage over temporal window (epochs) $T$, enabling scalable analysis of large networks across extended training periods.

\subsection{Memory Complexity Analysis}
\label{subsec:complexity_analysis_detailed}
Our sketching framework achieves memory reductions in two distinct scenarios: per-iteration 
training memory and persistent gradient monitoring. We analyze each separately to clarify 
where substantial savings occur.

\textit{Per-Iteration Training Memory.}
Standard backpropagation stores activation matrices $\mathbf{A}^{[\ell]} \in \mathbb{R}^{N_b \times d_{\ell}}$ requiring $\mathcal{O}(L \cdot N_b \cdot d_{\text{hidden}})$ memory. Our approach maintains sketches $\mathbf{X}_s^{[\ell]}, \mathbf{Y}_s^{[\ell]}, \mathbf{Z}_s^{[\ell]}$ with dimensions $d_{\text{hidden}} \times k$ where $k = 2r+1$, requiring $\mathcal{O}(L \cdot k \cdot d_{\text{hidden}})$ memory. For batch size $N_b = 128$ and adaptive rank $r = 2$-$16$ (giving $k = 5$-$33$), per-layer memory ratios range from $\frac{15}{128} \approx 0.12$ to $\frac{99}{128} \approx 0.77$, yielding 23-88\% memory reduction per iteration.

\textit{Persistent Gradient Monitoring Memory.}
Traditional gradient monitoring requires storing gradient matrices $\nabla_{\mathbf{W}^{[\ell]}} \mathcal{L} \in \mathbb{R}^{d_{\ell} \times d_{\ell-1}}$ across temporal window $T$, requiring $\mathcal{O}(L \cdot d_{\text{hidden}}^2 \cdot T)$ memory. Our approach maintains one set of EMA sketches requiring $\mathcal{O}(L \cdot k \cdot d_{\text{hidden}})$ memory, achieving reduction factor $\frac{k}{d_{\text{hidden}} \cdot T}$.

\section{Experiments and Results}
\label{sec:experiments_and_results}

We present comprehensive empirical evaluation of our sketching framework across two 
computational domains: image classification (MNIST and CIFAR-10) and scientific 
computing (physics-informed neural networks). Our experimental design addresses fundamental questions about the practical utility of sketched backpropagation: Can sketched gradient computation maintain training effectiveness? Where does the approach provide the most substantial benefits?

\subsection{Experimental Design and Methodology}
\label{subsec:experimental_design}
\subsubsection{Algorithmic Variants}

We design a controlled experimental framework that isolates the contributions of our sketching approach through three carefully configured variants:

\textbf{Standard Backpropagation}: Implements conventional gradient computation $\nabla_{\mathbf{W}^{[\ell]}} \mathcal{L} = (\mathbf{\delta}^{[\ell]})^{\top} \mathbf{A}^{[\ell-1]}$ using PyTorch's automatic differentiation. This baseline establishes performance and memory benchmarks for comparison.

\textbf{Sketched Backpropagation (Fixed Rank)}: Deploys Algorithm~\ref{alg:adaptive_sketched_backprop} with fixed parameters: sketch rank $r = 2$, EMA momentum $\beta = 0.95$, and dimensions $k = s = 2r + 1 = 5$. This configuration isolates the core sketching mechanism without adaptive components.

\textbf{Adaptive Sketched Backpropagation}: Implements the complete framework including dynamic rank adjustment with $r \in [2, 16]$, initial rank $r_0 = 2$, and the adaptation parameters from Algorithm~\ref{alg:adaptive_sketched_backprop}. This demonstrates autonomous memory-accuracy optimization.
\subsubsection{Network Architectures}

Our architectural choices maximize gradient computation impact while ensuring fair comparisons across methods:

\textbf{MNIST}: Four-layer MLP with 512-dimensional hidden layers, $\tanh$ activations, and 1.33M parameters. The uniform layer dimensions optimize sketching effectiveness while providing sufficient computational complexity for evaluating sketching performance.

\textbf{CIFAR-10}: Hybrid CNN-MLP with convolutional feature extraction followed by three 512-dimensional fully-connected layers. Sketching applies only to dense layers, demonstrating selective deployment capabilities.

\textbf{PINNs}: Four-layer network with 50-dimensional hidden layers for solving the 2D Poisson equation $-\Delta u = 4\pi^2 \sin(2\pi x)\sin(2\pi y)$ on $[0,1]^2$. This scientific computing application requires exact gradient computation for PDE residual 
evaluation, making it ideal for testing monitoring-only configurations.

All experiments use Adam optimization with a learning rate of $1e-3$ and batch size $N_b = 128$ to maintain consistent sketching matrix dimensions.

\textbf{Gradient Monitoring (MNIST)}: Two contrasting sixteen-layer MLPs with 1024 neurons in each hidden layer designed to exhibit different gradient pathologies. The ``healthy" configuration uses Kaiming initialization, ReLU activations, and Adam optimization, while the ``problematic" configuration employs Xavier initialization with small gain (0.5), tanh activations, and SGD optimization to induce vanishing gradients and training difficulties.

\subsection{Training Results}
\label{subsec:training_results}

\subsubsection{Convergence and Accuracy Preservation}
Across tested domains, sketched backpropagation demonstrates a predictable accuracy-memory tradeoff. MNIST classification (Figure~\ref{fig:mnist_results}) shows sketched methods achieve 93-95\% test accuracy compared to 98\% for standard backpropagation within 50 epochs. This performance gap reflects the gradient approximation error bounded by Theorem~\ref{thm:gradient_reconstruction_error}. The similar convergence trajectories indicate that sketch-based gradient computation preserves essential optimization structure despite the approximation.

\begin{figure}[ht]
    \centering
    \includegraphics[width=0.45\textwidth]{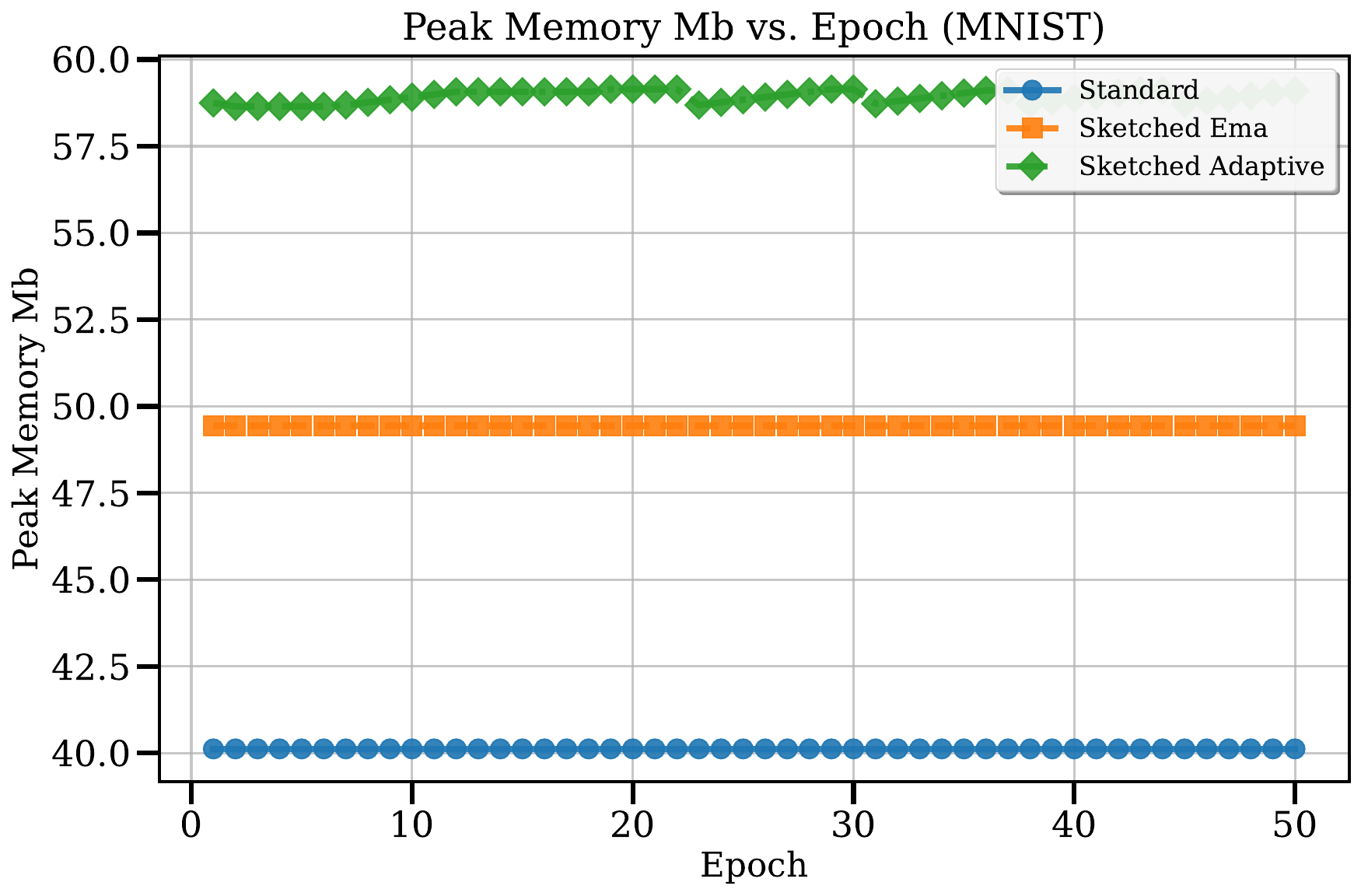}
    \includegraphics[width=0.45\textwidth]{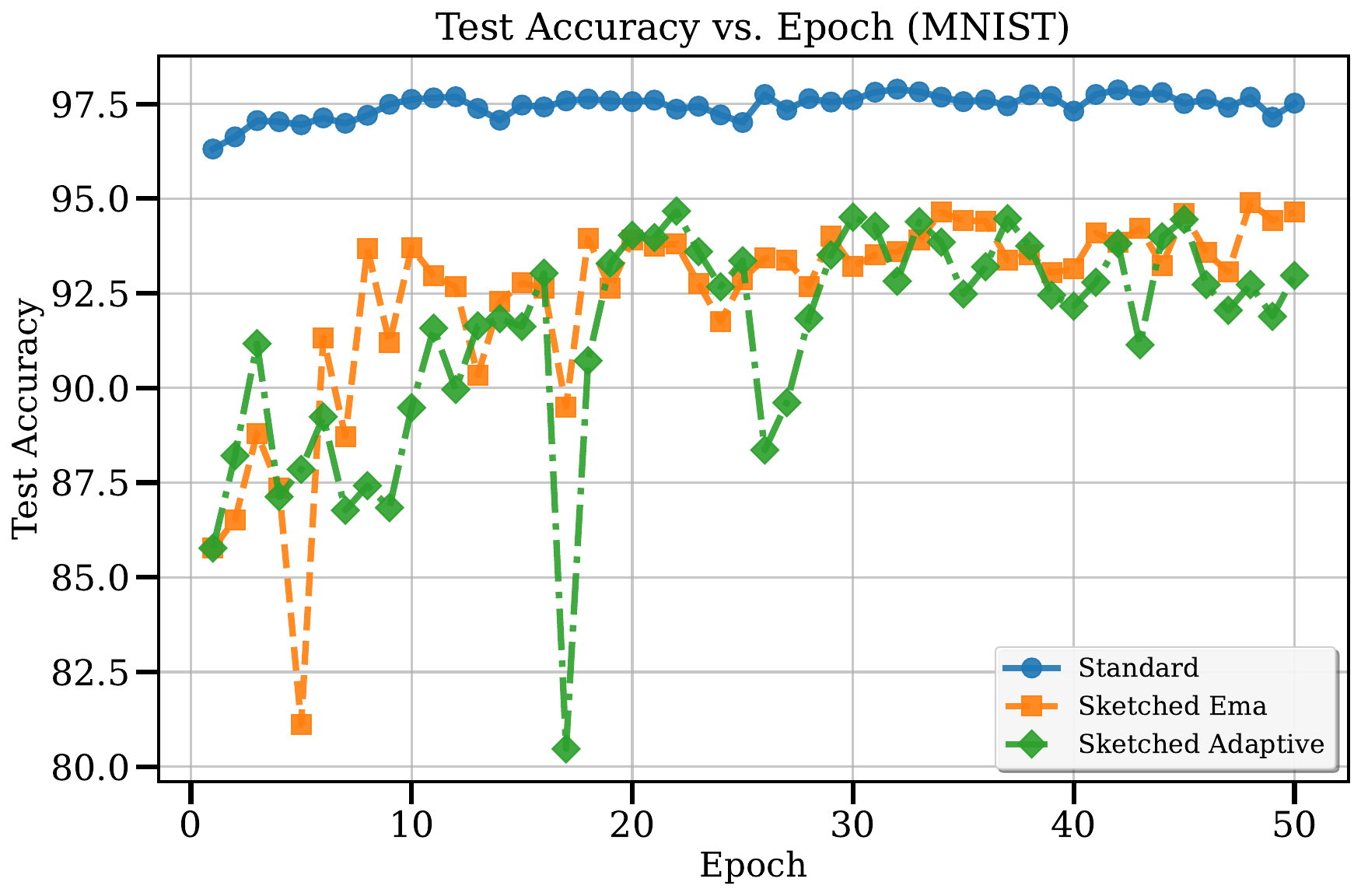}
    \caption{MNIST classification results: (Left) Peak memory usage comparison across methods. 
(Right) Training accuracy showing 3-5\% performance gap with preserved convergence dynamics.}
    \label{fig:mnist_results}
\end{figure}

CIFAR-10 experiments (Figure~\ref{fig:cifar10_results}) demonstrate effective gradient approximation in the hybrid CNN-MLP architecture despite increased complexity. Both standard and sketched backpropagation achieve 80\% test accuracy, indicating that selective sketching on dense layers while maintaining convolutional feature extraction preserves training effectiveness. The convergence trajectories remain similar, validating that sketch-based gradient computation 
can be selectively applied to fully-connected layers without compromising overall training 
effectiveness.

\begin{figure}[ht]
    \centering
    \includegraphics[width=0.45\textwidth]{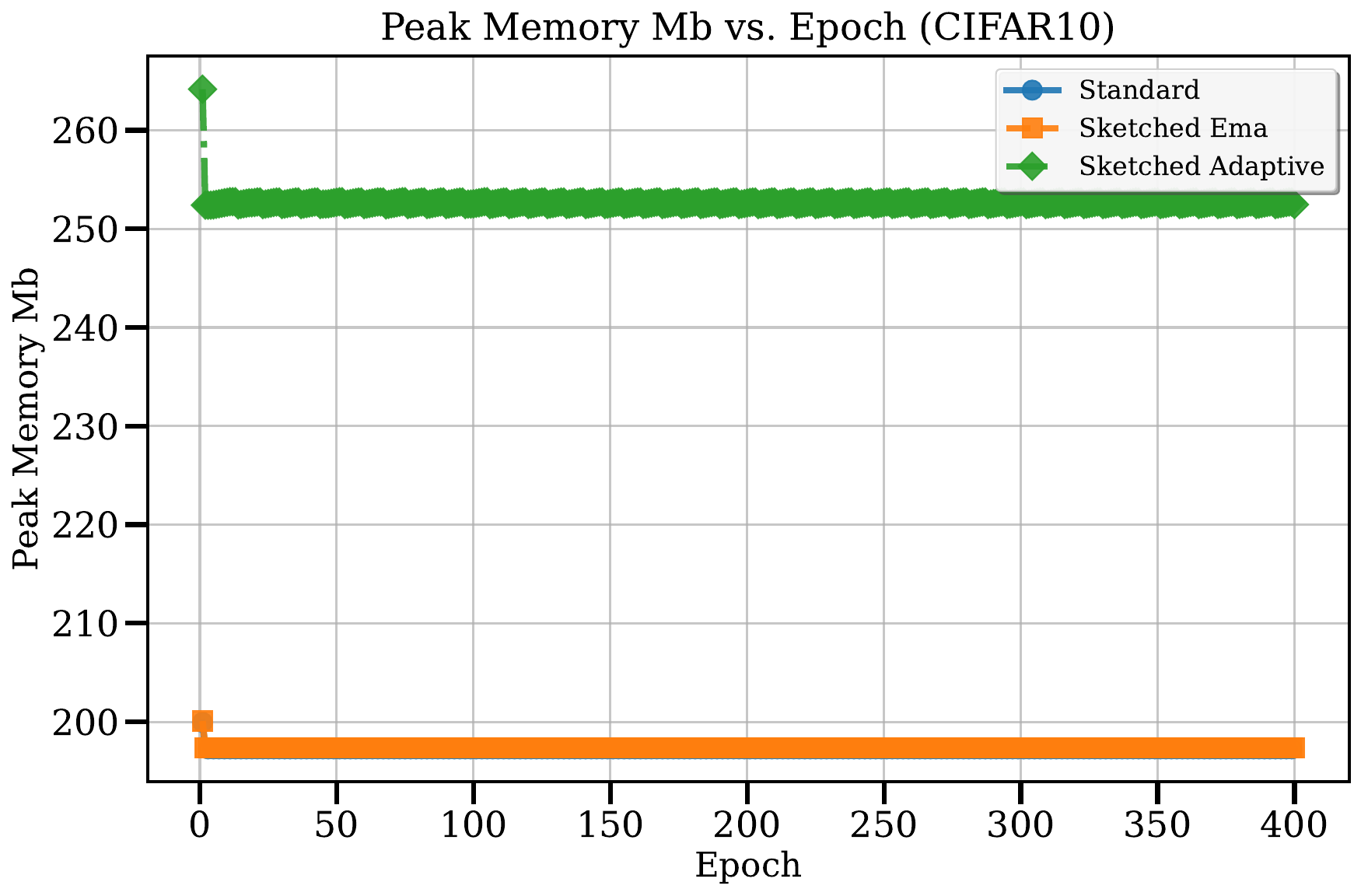}
    \includegraphics[width=0.45\textwidth]{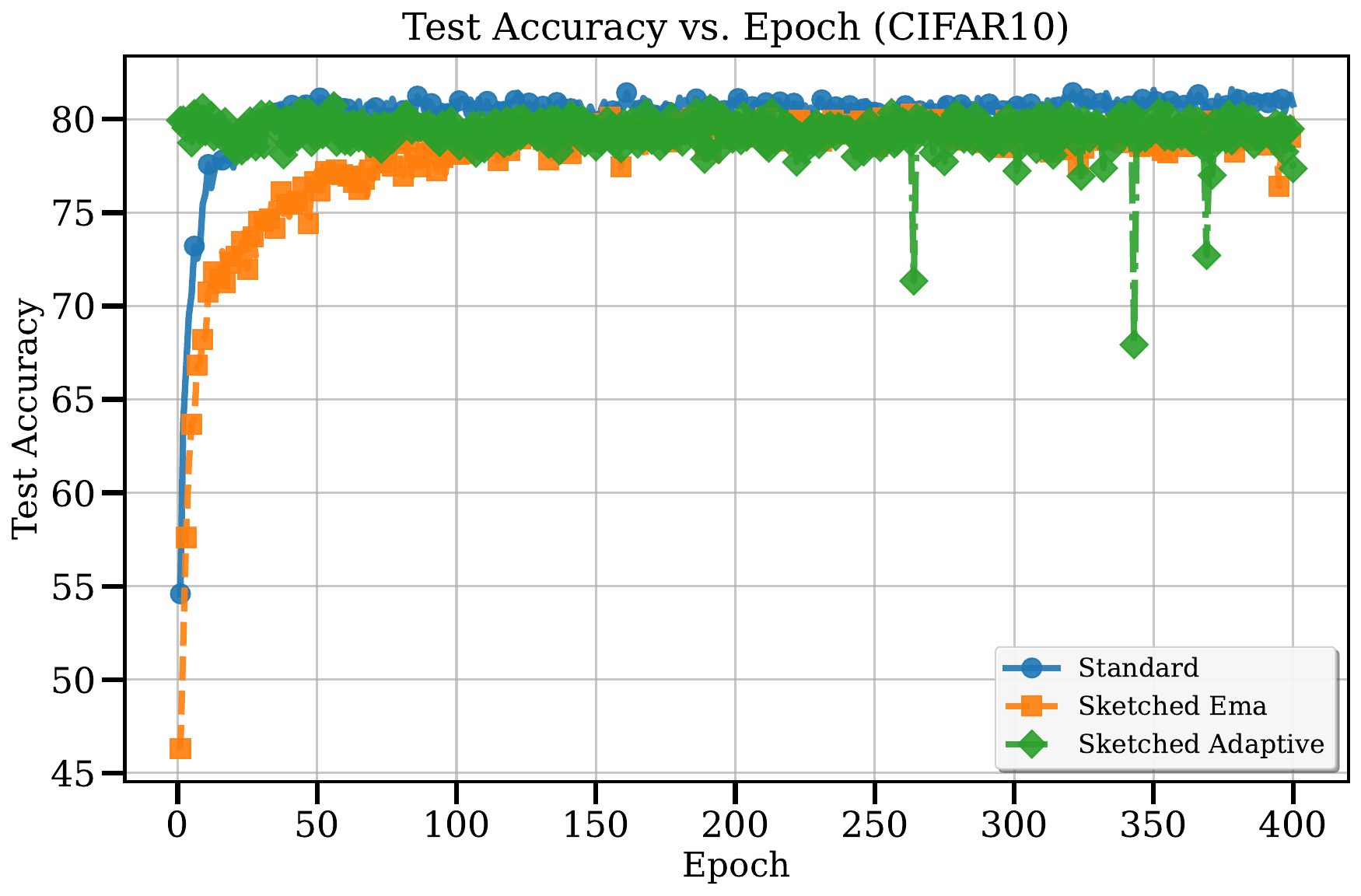}
    \caption{CIFAR-10 results for hybrid CNN-MLP architecture: (Left) Memory usage with sketched fully-connected layers. (Right) Accuracy preservation demonstrating effective integration with convolutional components.}
    \label{fig:cifar10_results}
\end{figure}
Theorem~\ref{thm:gradient_reconstruction_error} predicts accuracy degradation proportional to $\tau_{r+1}$. Experimental results validate this prediction across different deployment strategies: with low rank 
applied to all hidden layers, MNIST shows 3-5\% accuracy reduction, demonstrating the theoretical tradeoff. In contrast, CIFAR-10 maintains equivalent accuracy (80\% for both standard and sketched variants) through selective sketching of fully-connected layers. This demonstrates that strategic layer selection can preserve accuracy while achieving memory savings. The controllable nature of this tradeoff through both rank selection and layer-wise deployment confirms the practical utility of our theoretical framework.

\subsubsection{Physics-Informed Neural Networks
}

Physics-informed neural networks require computing spatial-temporal derivatives for PDE residual evaluation. These derivatives necessitate exact gradient computation during the loss calculation phase. For such applications, we employ a monitoring-only configuration: maintaining standard backpropagation for parameter updates while accumulating sketches through forward hooks for diagnostic purposes.

Figure~\ref{fig:pinns_results} demonstrates that this monitoring framework preserves solution 
quality. All variants: standard backpropagation, fixed-rank sketching ($r=2$), and adaptive 
sketching, achieve identical $L^2$ relative error of 0.31 for the 2D Poisson equation. The sketch accumulation introduces only 0.57 MB memory overhead, confirming that comprehensive gradient monitoring can be achieved without compromising physics constraint satisfaction or solution accuracy.
Solution quality analysis (Figure~\ref{fig:pinn_solutions}) reveals that all methods reproduce the analytical solution with equivalent fidelity. 

\begin{figure}[ht]
    \centering
    \includegraphics[width=0.45\textwidth]{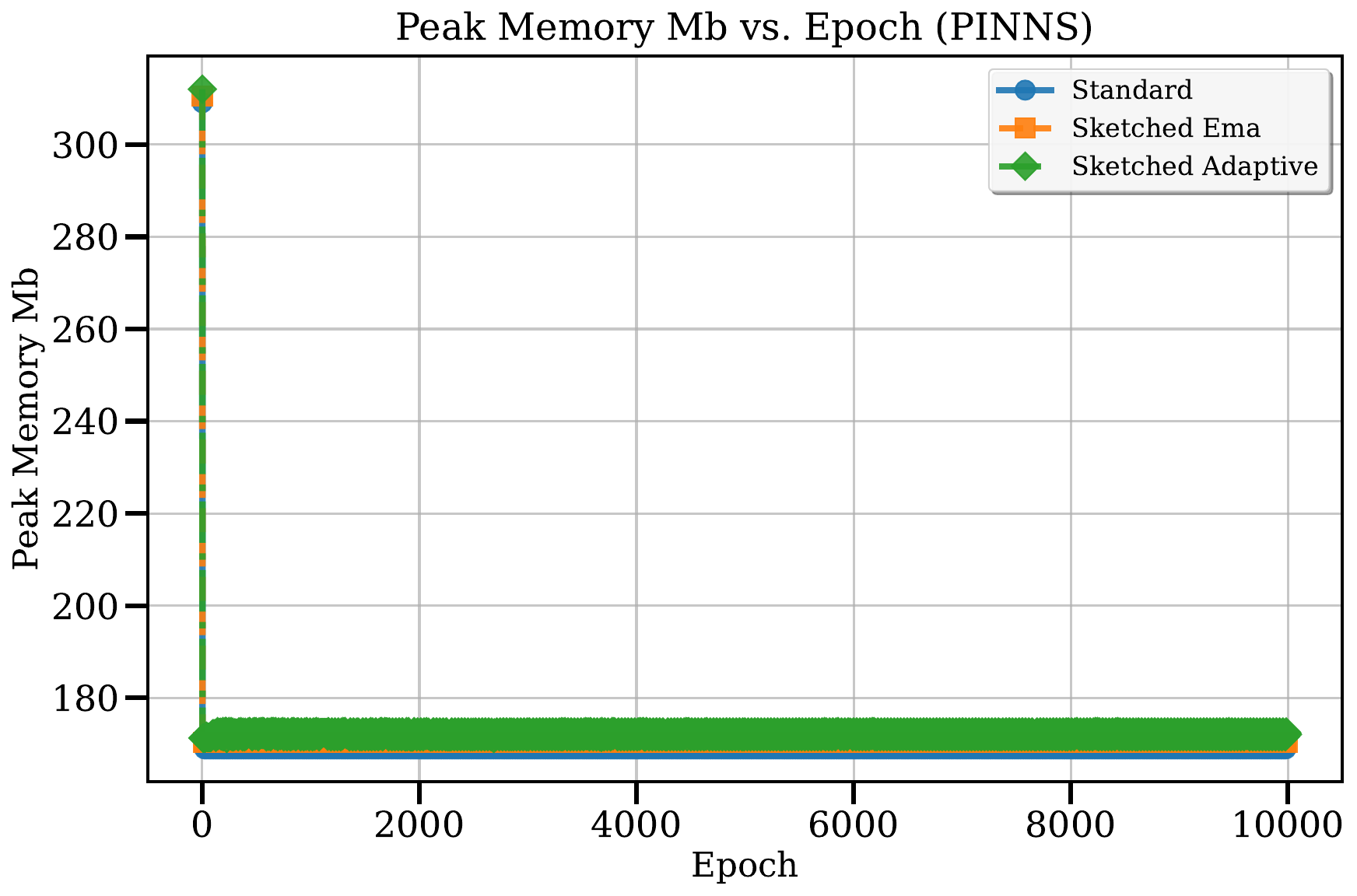}
    \includegraphics[width=0.45\textwidth]{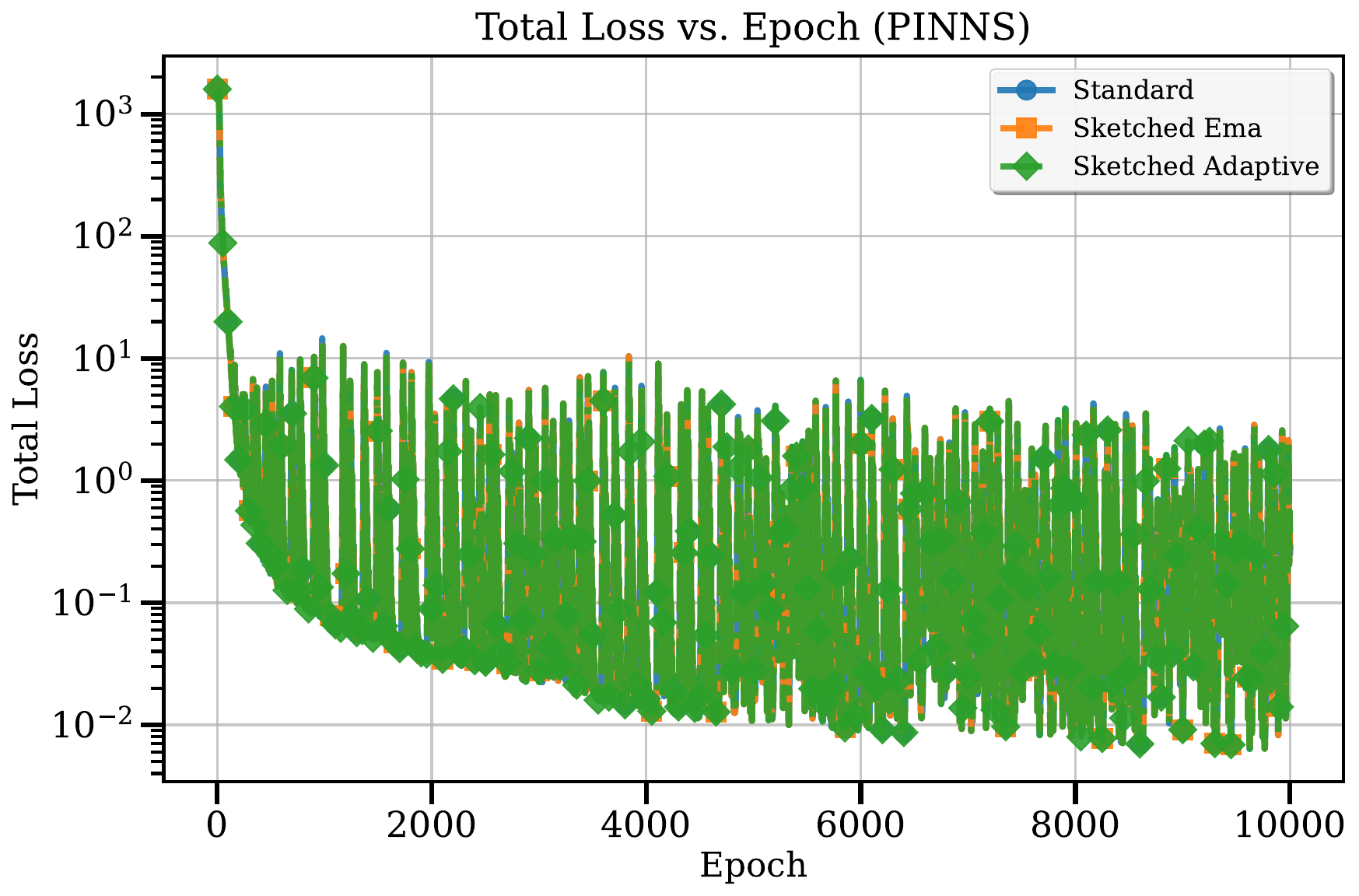}
    \caption{PINN training results with sketch-based monitoring: (Left) Peak memory usage showing minimal overhead (0.57 MB) from sketch storage for monitoring. (Right) Total loss convergence demonstrating identical training dynamics across standard backpropagation and sketch-based monitoring variants, validating that comprehensive gradient monitoring can be achieved without compromising physics constraint satisfaction. All methods achieve $L^2$ relative error of 0.31 on testing points.}
    \label{fig:pinns_results}
\end{figure}

\begin{figure}[ht]
    \centering
    \begin{minipage}[t]{0.32\textwidth}
        \vspace*{5.1cm}
        \centering
        \includegraphics[width=\linewidth]{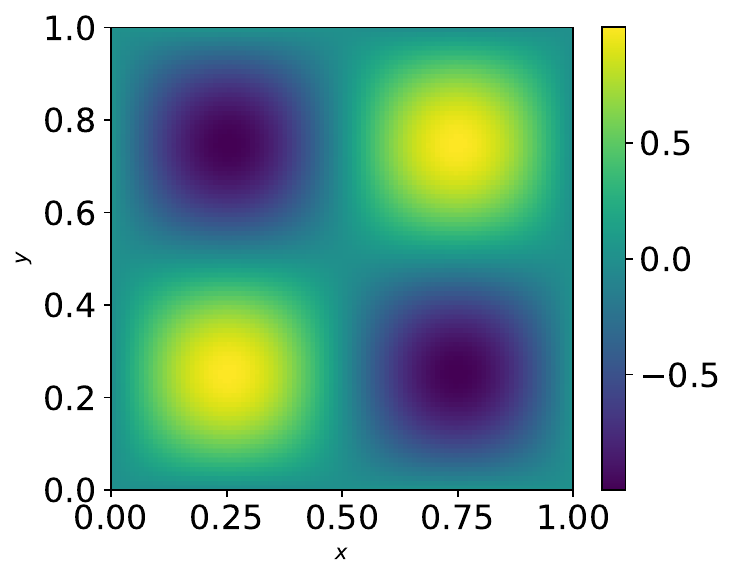}
        
        \textbf{Exact}
        \vfill
    \end{minipage}%
    \hfill
    \begin{minipage}[t]{0.64\textwidth}
        \centering
        \textbf{Standard}\\
        \includegraphics[width=\linewidth]{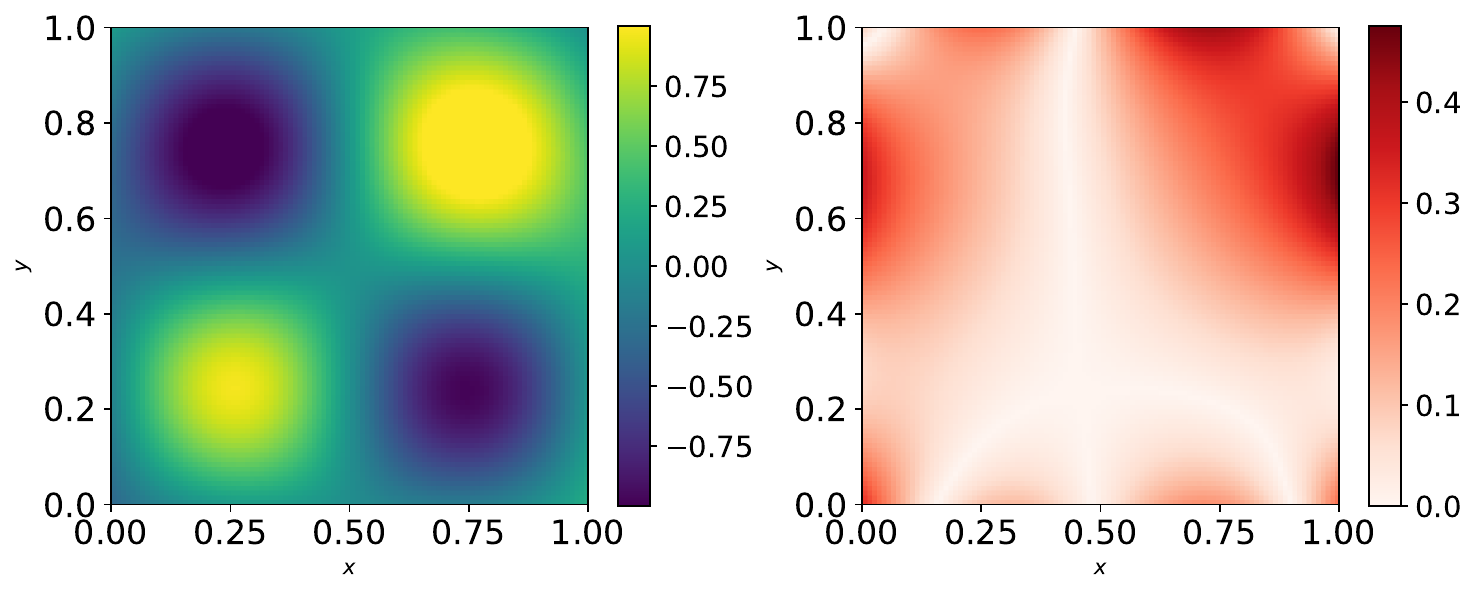}\\[0.7em]
        
        \textbf{Sketched – Fixed Rank}\\
        \includegraphics[width=\linewidth]{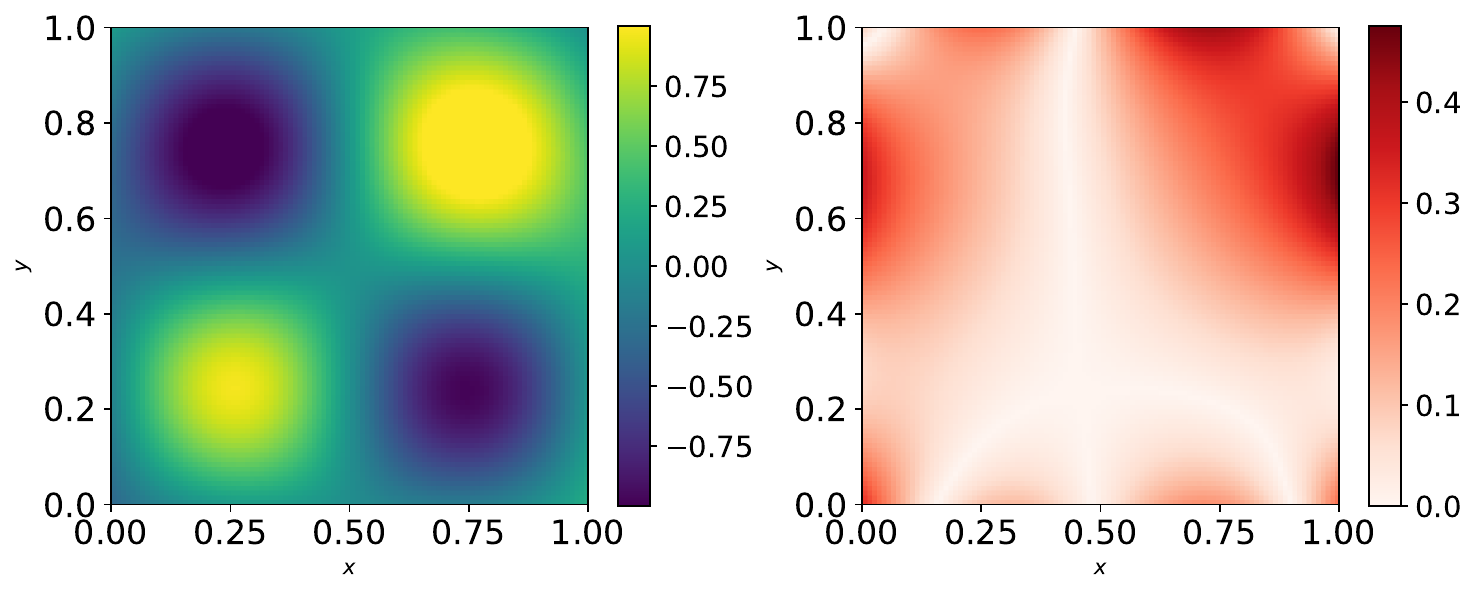}\\[0.7em]
        
        \textbf{Sketched – Adaptive Rank}\\
        \includegraphics[width=\linewidth]{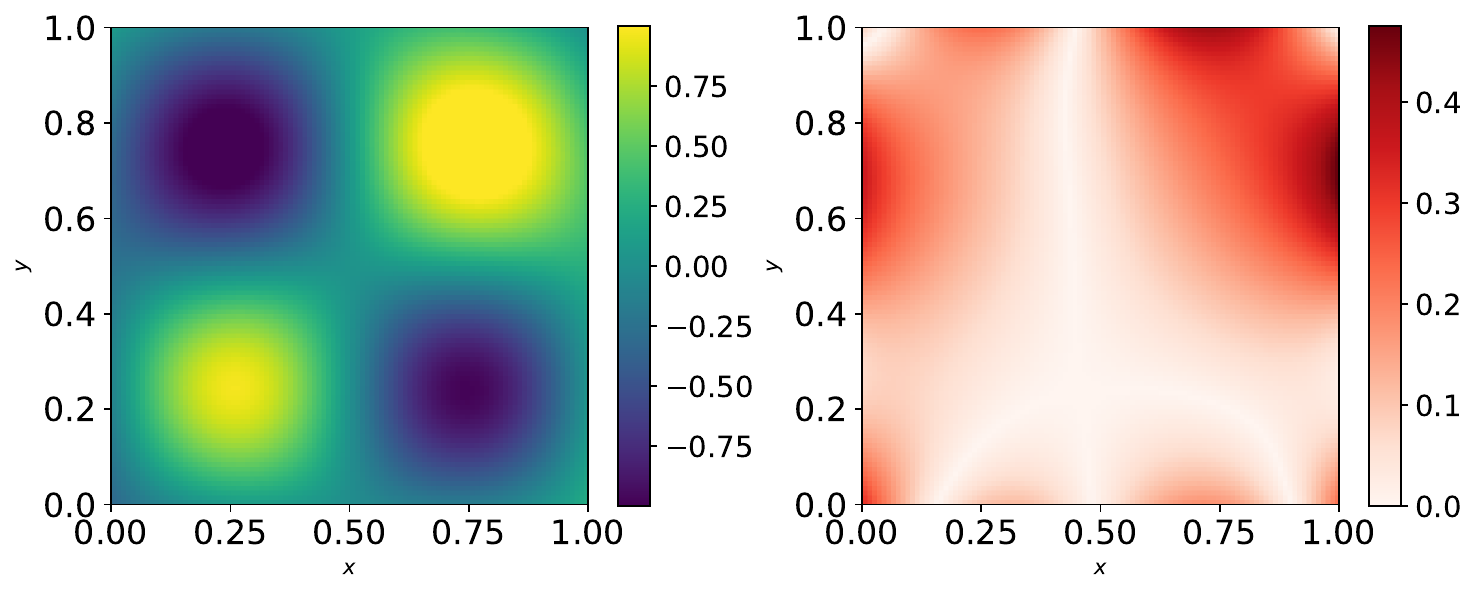}

    \end{minipage}
        \begin{tabular}{ccc}
            \phantom{\makebox[0.32\linewidth][c]{\textbf{Exact}} }
            &
            \makebox[0.32\linewidth][c]{\textbf{Predicted}} &
            \makebox[0.32\linewidth][c]{\textbf{Absolute Error}} \\
        \end{tabular}

    \caption{PINN solution quality comparison: (Left) Exact solution, (Right, top to bottom) Standard backpropagation, Fixed-rank sketching, and Adaptive sketching. Each right column shows the predicted solution and the corresponding absolute error. All methods achieve $L^2$ relative error of 0.31 on testing points.}
    \label{fig:pinn_solutions}
\end{figure}

\subsection{Gradient Monitoring Application}
\label{subsec:gradient_monitoring}
Beyond direct training applications, our sketching framework enables comprehensive gradient 
monitoring with dramatic memory reductions. We demonstrate this capability through two 
contrasting sixteen-layer MLPs with 1024-dimensional hidden layers on MNIST classification: 
a ``healthy'' configuration using Kaiming initialization, ReLU activations, and Adam 
optimization, versus a ``problematic'' configuration employing Kaiming initialization with 
strong negative bias ($b=-3.0$), ReLU activations, and SGD optimization to induce training difficulties. Both networks use sketch rank $r=4$ (sketch dimensions 
$k=s=9$) with EMA momentum $\beta=0.9$. 

Figure~\ref{fig:gradient_monitoring} demonstrates comprehensive monitoring capabilities 
distinguishing healthy from problematic training dynamics. Training loss and accuracy curves 
reveal dramatic performance differences: the healthy network rapidly learns, achieving 
95\%+ accuracy within 10 epochs, while the problematic network exhibits complete training 
stagnation with accuracy fluctuating around random performance (10-11\%) throughout training.

Gradient norm analysis using $\|\mathbf{Z}_s^{[\ell]}\|_F$ captures distinct training 
dynamics. The healthy network shows gradient magnitudes ranging from $10^2$ to $10^4$ 
throughout training, indicating active parameter updates, while the problematic network 
exhibits constant gradient norms around $10^2$, confirming optimization stagnation. These 
sketch-based estimates provide effective proxies for full gradient magnitudes without 
materializing complete gradient matrices.

Gradient diversity measurement through stable rank estimation provides the most 
discriminative diagnostic signal. We compute 
$\text{rank}_{\text{stable}}(\mathbf{Y}_s^{[\ell]}) = \|\mathbf{Y}_s^{[\ell]}\|_F^2/\|\mathbf{Y}_s^{[\ell]}\|_2^2$ 
efficiently from Y-sketches. The healthy network achieves stable rank of $9.0$, indicating 
gradients span the full sketch subspace (matching $k=9$), while the problematic network 
shows only $2.9$, revealing severely collapsed gradient diversity. This collapse directly 
confirms the training pathology through reduced gradient dimensionality, demonstrating how 
sketch-based metrics detect subtle training failures invisible to loss curves alone.

The critical advantage lies in memory efficiency. Traditional gradient monitoring requires 
storing complete gradient matrices $\nabla_{\mathbf{W}^{[\ell]}} \mathcal{L} \in 
\mathbb{R}^{d_{\ell} \times d_{\ell-1}}$ at multiple temporal checkpoints throughout 
training. For our architecture with $L=16$ layers and $d_{\text{hidden}}=1024$, each 
checkpoint requires 64 MB of gradient storage. Monitoring over a temporal window of $T=5$ 
epochs (one checkpoint per epoch) demands $320$ MB total storage with complexity 
$\mathcal{O}(L \cdot d_{\text{hidden}}^2 \cdot T)$.

In contrast, our sketched approach maintains a single set of EMA-updated sketches requiring 
only $\mathcal{O}(L \cdot d_{\text{hidden}} \cdot r)$ storage independent of monitoring 
duration. The sketch storage totals 1.7 MB regardless of temporal window length, achieving 
99\% memory reduction compared to traditional monitoring over $T=5$ epochs 
(320 MB $\rightarrow$ 1.7 MB). This reduction grows with monitoring duration.

\begin{figure}[ht]
    \centering
    \includegraphics[width=\textwidth]{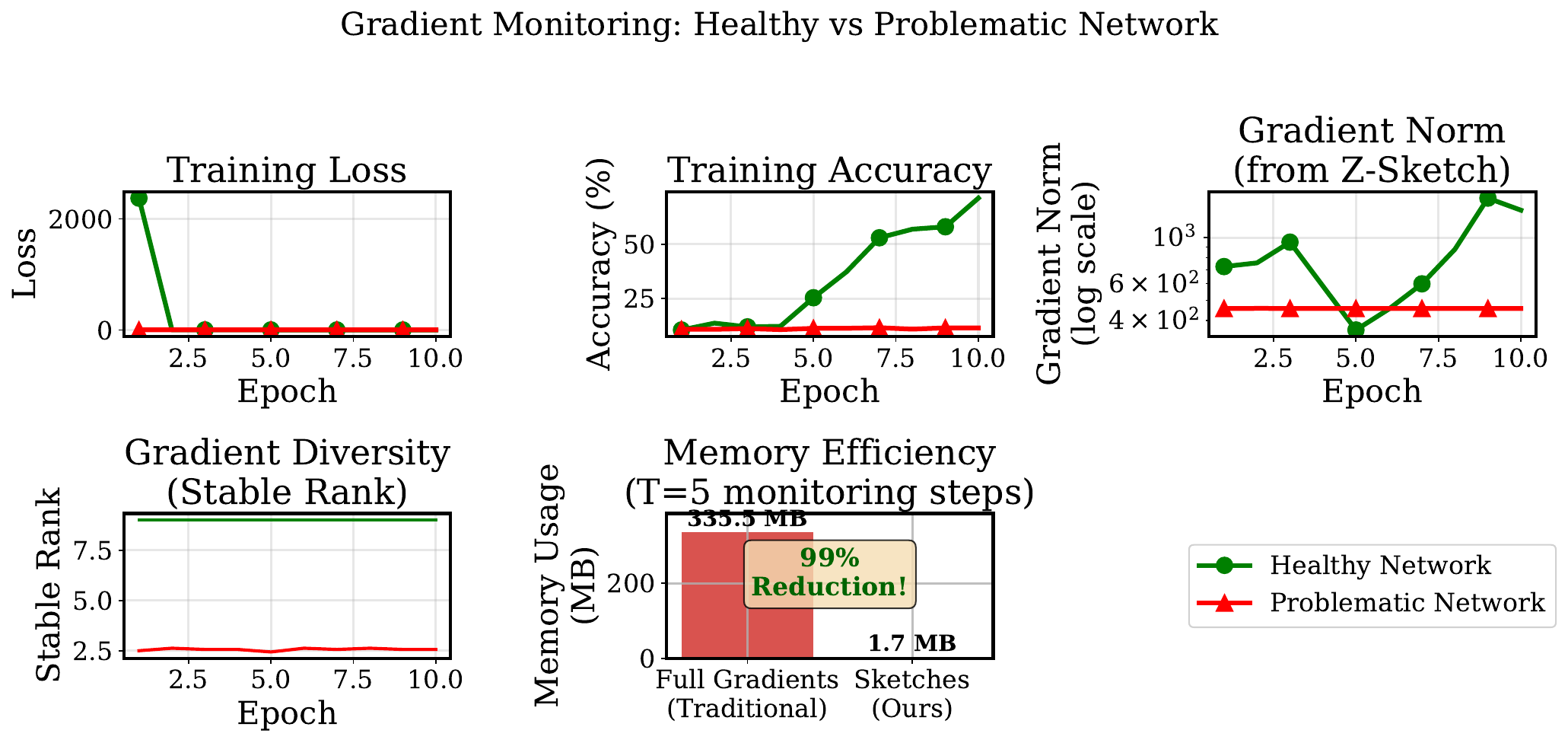}
    \caption{Gradient monitoring demonstration comparing healthy and problematic network 
configurations on MNIST. Both sixteen-layer networks (1024 neurons in each hidden layer) use 
sketch rank $r=4$. }
    \label{fig:gradient_monitoring}
\end{figure}



\subsection{Memory Analysis and Practical Implications}
\label{subsec:memory_analysis}
\phantom{new line}

\noindent \textit{Theory versus Practice.}
Our evaluation reveals significant discrepancy between theoretical and practical memory savings during direct training. While theoretical analysis predicts reduction from $\mathcal{O}(L \cdot d_{\text{hidden}}^2 \cdot T)$ to $\mathcal{O}(L \cdot d_{\text{hidden}} \cdot r + N_b \cdot r)$, empirical results show modest gains offset by implementation overhead. PyTorch's sophisticated memory management already optimizes gradient storage through dynamic allocation and reuse, while our sketching framework introduces overhead from EMA matrices and projection buffers.

However, gradient monitoring applications achieve dramatic benefits precisely where genuine memory bottlenecks exist. Our sixteen-layer MNIST experiment (Section~\ref{subsec:gradient_monitoring}) demonstrates the distinction: sketch-based monitoring requires only $1.7$ MB regardless of monitoring duration, compared to $335$ MB for traditional monitoring over $T=5$ epochs (99\% reduction). The elimination of temporal factor $T$ enables continuous gradient analysis over extended training periods—maintaining full diagnostic capability across hundreds of epochs within constant memory. For PINNs requiring exact gradients for physics constraints, monitoring-only configuration adds minimal $0.57$ MB overhead while preserving solution quality.

\noindent \textit{Optimal Application Domains.}
Our findings identify where sketching provides maximum benefit:

\textbf{Extended Gradient Monitoring}: Applications requiring gradient analysis over temporal windows benefit from eliminating the $T$ factor.

\textbf{Physics-Constrained Training}: Scenarios demanding exact gradients for constraints (PINNs, optimal control) while needing memory-efficient diagnostics.

\textbf{Large-Scale Configurations}: Wider networks ($d_{\text{hidden}} \geq 2048$) and deeper architectures ($L \geq 50$) where gradient storage dominates memory consumption.

The combination of preserved diagnostic capability with dramatic memory reduction establishes sketched backpropagation as a valuable tool for neural network analysis in memory-constrained production environments.

\section{Conclusion}
\label{sec:conclusion}

We have presented the first adaptation of control-theoretic matrix sketching to neural network gradient computation. Our EMA-based sketching framework with adaptive rank adjustment enables memory-efficient gradient analysis while preserving training effectiveness.

Experimental validation across MNIST, CIFAR-10, and PINNs reveals application-specific effectiveness. Gradient monitoring applications achieve 99+\% memory reduction over extended temporal windows, enabling continuous training diagnostics previously infeasible due to storage constraints. For applications requiring exact gradients (PINNs), monitoring-only deployment adds minimal overhead while maintaining solution quality. This opens new possibilities for understanding and debugging neural network optimization in production environments.

Future work will explore transformer architecture adaptation, develop theoretical guarantees for EMA approximation quality under stochastic mini-batch dynamics, investigate selective per-layer sketching strategies, and examine integration with complementary memory optimization techniques for large-scale neural network analysis.

\bibliographystyle{plain}
\bibliography{references}

\end{document}